\documentclass[nohyperref]{article}
\usepackage{import}

\usepackage[table]{xcolor}
\usepackage{float}
\newcommand{\defeq}{\vcentcolon=}

\newcommand{\dist}{\mathcal{D}}
\newcommand{\loss}{\ell}
\newcommand{\cost}{\mathcal{L}}
\newcommand{\E}{\mathbb{E}}
\newcommand{\gap}{\mathcal{G}}
\newcommand{\samples}{\mathcal{S}}
\usepackage{pdfpages}
\usepackage{xcolor}

\newcommand{\Prob}{\text{Pr}}
\newcommand{\interior}{\text{int}}
\newcommand{\closure}{\text{cl}}

\newcommand{\cX}{\mathcal{X}}
\newcommand{\cY}{\mathcal{Y}}
\newcommand{\cN}{\mathcal{N}}

\newcommand{\adv}{\text{AT}}
\newcommand{\shift}{\text{SHIFT}}
\newcommand{\ERM}{\text{ERM}}
\newcommand{\test}{\text{test}}
\newcommand{\KL}{\text{KL}}
\newcommand{\LP}{\text{LP}}
\newcommand{\TV}{\text{TV}}
\newcommand{\HR}{\text{HR}}
\newcommand{\st}{\text{s.t.}}
\newcommand{\supp}{\text{supp}}
\newcommand{\github}[1]{%
   \href{#1}{\faGithubSquare}%
}
\usepackage{pgfplots}
\usepackage{tabularx}
\usepackage{fontawesome5}
\usepgfplotslibrary{groupplots,dateplot}
\usepackage{svg}
\usepackage{float}
\usepackage{filecontents}

\usepackage{microtype}
\usepackage{graphicx}
\usepackage{subfigure}
\usepackage{booktabs} 

\usepackage[noline,ruled]{algorithm2e}

\usepackage{hyperref}

\usepackage{acronym}
\acrodef{AT}[AT]{adversarial training}
\acrodef{ERM}[ERM]{empirical risk minimization}
\acrodef{DRO}[DRO]{distributionally robust optimization}
\acrodef{HR}[HR]{holistic robust}
\acrodef{KL}[KL]{Kullback-Leibler}
\acrodef{LP}[LP]{L\'evy-Prokhorov}
\acrodef{TV}[TV]{Total Variations}
\acrodef{PGD}[PGD]{projected gradient descent}
\acrodef{IID}[IID]{independent identically distributed}
\acrodef{WDRO}[WDRO]{Wasserstein distributionally robust optimization}
\acrodef{SOTA}[SOTA]{state of the art}
\acrodef{SGD}[SGD]{stochastic gradient descent}

\usepackage[accepted]{icml2023}

\usepackage{amsmath}
\usepackage{amssymb}
\usepackage{mathtools}
\usepackage{amsthm}
\usepackage{amsfonts}
\usepackage{bbm}

\theoremstyle{plain}
\newtheorem{theorem}{Theorem}[section]

\newtheorem{lemma}[theorem]{Lemma}

\theoremstyle{definition}

\newtheorem{assumption}[theorem]{Assumption}

\theoremstyle{remark}

\usepackage{amsmath}
\DeclareMathOperator*{\argmax}{arg\,max}
\DeclareMathOperator*{\argmin}{arg\,min}

\icmltitlerunning{Certified Robust Neural Networks: Generalization and Corruption Resistance}

\renewcommand{\tfrac}[2]{#1/#2}
\renewcommand{\d}{\text{d}}

\begin{document}

\twocolumn[
\icmltitle{Certified Robust Neural Networks: Generalization and Corruption Resistance}



\icmlsetsymbol{equal}{*}


\begin{icmlauthorlist}
\icmlauthor{Amine Bennouna}{mit}
\icmlauthor{Ryan Lucas}{mit}
\icmlauthor{Bart Van Parys}{mit}
\end{icmlauthorlist}

\icmlaffiliation{mit}{Operations Research Center, Massachusetts Institute of Technology, Cambridge, MA, USA}

\icmlcorrespondingauthor{Amine Bennouna}{amineben@mit.edu}
\icmlcorrespondingauthor{Ryan Lucas}{ryanlu@mit.edu}
\icmlcorrespondingauthor{Bart Van Parys}{vanparys@mit.edu}

\icmlkeywords{Neural Networks, Adversarial Attacks, Data Poisoning, Distributional Robustness}

\vskip 0.3in
]



\printAffiliationsAndNotice{}  

\begin{abstract}
Recent work have demonstrated that robustness (to ``corruption'') can be at odds with generalization.
Adversarial training, for instance, aims to reduce the problematic susceptibility of modern neural networks to small data perturbations.
Surprisingly, overfitting is a major concern in adversarial training despite being mostly absent in standard training.
We provide here theoretical evidence for this peculiar ``robust overfitting'' phenomenon.
Subsequently, we advance a novel distributionally robust loss function bridging robustness and generalization.
We demonstrate both theoretically as well as empirically the loss to enjoy a certified level of robustness against two common types of corruption|data evasion and poisoning attacks|while ensuring guaranteed generalization.
We show through careful numerical experiments that our resulting holistic robust (HR) training procedure yields \acl{SOTA} performance.
Finally, we indicate that HR training can be interpreted as a direct extension of adversarial training and comes with a negligible additional computational burden. 
A ready-to-use python library implementing our algorithm is available at \url{https://github.com/RyanLucas3/HR_Neural_Networks}.
\end{abstract}

\section{Introduction}
\label{sec:intro}
Recent work has shown that many modern neural networks are vulnerable to data corruption putting into question the reliability and performance of such models at deployment time.
Robustness has in recent years, next to out-of-sample generalization, become a central research topic in the machine learning community.
The key question is how to develop machine learning models which are reliable and perform well despite being trained on corrupted training data and evaluated on perturbed test data.

Data corruption can be due to multiple causes ranging from benign imprecise data collection to malicious adversaries.
Two types of corruption, each associated with a distinct cause, stand out in prior work which we briefly discuss.

\paragraph{Evasion Corruption.} This type of corruption results in small, often imperceptible, perturbations of all the test or train instances and may fool a naively trained model to suffer poor performance at deployment time.
\citet{goodfellow2014explaining} study evasion attacks at test time and show that neural network models can suffer a severe decrease in accuracy when test instances are subjected to an imperceptible amount of noise.
This pioneering work spurred on a rich literature on attacks \cite{tramer2020adaptive} as well as defense strategies such as the popular adversarial training \cite{madry2018towards, bai2021recent}.
Evasion corruption need not necessarily be caused by an adversary or even take place at test time.
Indeed, noise is present in training data as well often as a result of measurement imperfections.
Examples include street noise and background chatter in speech recognition \cite{li2014overview, mitra2017robust} or impulse noise and blur in image classification \cite{hendrycks2019benchmarking,hendrycks2021many}.

\paragraph{Poisoning Corruption.} In this type of corruption, a certain fraction of the training data points is drastically changed.
This alteration of the training dataset may then mislead the learning in selecting a ``bad'' model.
Clearly, if a large majority of data points is affected, the data set is wholly corrupted and there is no point in hoping for a good performance.
Hence, poisoning corruption is assumed to affect only a small fraction of the data points.
One example is label flipping in classification problems, where a small portion of the training data points are wrongly labeled or an adversary can select a small number of labels to flip before training time \cite{biggio2012poisoning}.
Poisoning corruption has also been shown to considerably degrade performance at deployment time.
For instance, \citet{nelson2008exploiting} shows that altering only $1\%$ of email spam labels can completely fool commercial spam filters. Several related attacks and defense mechanisms have been identified in follow-up work \cite{goldblum2022dataset}.


\paragraph{Generalization.} Beyond protecting against both discussed types of attacks, an equally important challenge is out-of-sample performance.
A model generalizes if it provides protection against \textbf{statistical error} which we take here to mean the difference between the finite number of data points provided during training time as opposed to the full data distribution.
Indeed, all models must avoid \textit{overfitting} to a particular training set and generalize to the test data.
Several classical approaches have been developed to address statistical error such as $\ell_1,\ell_2$ regularization, early stopping, and data augmentation. Such models seek not only a good expected performance under the randomness of the samples but also low variance. To generalization well, models need to avoid being overly sensitive to a particular training data and rather seek ``satisfactory'' performance on all ``likely'' data sets of the task at hand.

Generalization and robustness to each type of corruption have been extensively studied in isolation. However, the interaction between corruption protection and generalization has been largely ignored: existing ``robust'' approaches typically dismiss statistical error. Recent work has shown that robustness (to corruption) can be at odds with generalization, highlighting the need to consider corruption and statistical error simultaneously.
\citet{schmidt2018adversarially} and \citet{rice2020overfitting} exhibit theoretical and empirical evidence that \ac{AT}---which provides robustness against evasive corruption---suffers from notably worse generalization than natural training.
\citet{rice2020overfitting} show in particular that the surprising generalization properties of neural networks, such as improved out-of-sample performance with larger models, vanishes under \ac{AT}. \citet{schmidt2018adversarially} indicate that ``understanding the interactions between robustness, classifier model, and data distribution from the perspective of generalization is an important direction for future work''.
Hence, while \ac{AT} aims to improve generalization performance by protecting against evasive corruption, it seems inadvertently to limit generalization by exacerbating the adverse effect of statistical error.
This interesting phenomenon, dubbed \textit{robust overfitting}, and its remedies is the topic of several recent papers \cite{wu2020adversarial,li2022understanding,yu2022understanding}. 
Similar observations were also noted in the case of poisoning, where poisoning corruption exacerbates overfitting \cite{zhang2018generalized,hendrycks2019using}.

Besides generalization issues, existing robust approaches typically protect against one single type of corruption exclusively, either evasion or poisoning. However, in practical settings, we typically are unaware of what type of corruption, if any, plagues the data. Unsurprisingly, a model designed for a specific type of corruption risks considerably poor performance when deployed in a setting with a different type of corruption.
Hence, to develop practical robust models, it is important to protect against both types of corruption, and their potential combination.

\subsection{Contributions}

In this paper, we study generalization when learning with data affected by evasion and/or poisoning corruption on top of statistical error simultaneously.

We first provide theoretical evidence of the significance of the interaction between corruption and statistical error.
We show that the generalization gap of \ac{AT} can be decomposed into an \ac{ERM} generalization gap and a positive term which we interpret as an \textit{adversary shift} (Section \ref{sec: adv and generalization}).
This result provides theoretical evidence to the robust overfitting phenomenon observed empirically by \citet{rice2020overfitting} and indicates that \ac{AT} indeed suffers from worse generalization than \ac{ERM}.
This theoretical observation illustrates that protecting against corruption only while ignoring statistical error can inadvertently yield worse generalization performance.

Second, we tackle the problem of robustness against both types of corruption and statistical error simultaneously.
We do so by designing and optimizing a provable ``tight'' upper bound on the \textit{test loss} when learning with data subject to corruption.
This is in contrast to most prior work which is either empirically motivated or, like \ac{AT}, merely optimizes empirical proxies to such upper bounds|dismissing statistical error.
To design this upper bound, we construct a set around the corrupted training distribution which contains the testing distribution with high probability.
A valid upper bound on the test loss can then be taken as the worst-case expectation of the loss over all distributions in this ambiguity set.
We show that training a neural network model based on such associated \ac{DRO} objective function can be interpreted as a natural generalization to standard \ac{AT} which protects against the two discussed corruption sources simultaneously and particularly enjoys strong generalization.
Our \ac{HR} objective function has three parameters $\alpha\in [0,1]$, $r\geq 0$ and $\cN$, and is guaranteed to be an upper bound on the test performance with high probability $1-e^{-rn+O(1)}$ when less then a fraction $\alpha$ of all $n$ samples are tampered by poisoning, and the evasion corruption is bounded within the set $\cN$.
The latter guarantee certifies robustness for \textit{any} desired level of evasion attacks ($\cN$) and poisoning attacks ($\alpha$): our out-of-sample deployment performance is at least as good as the estimated in-sample performance with high probability.

Finally, we propose a practical \ac{HR} loss training algorithm for neural networks with negligible additional computational burden compared to classical \ac{AT}.
We conduct careful numerical experiments which illustrate the efficacy of \ac{HR} training on both MNIST and CIFAR-10 datasets in all possible corruption settings: clean, affected by poisoning and/or evasive corruption. By using validation on the protection parameters $\alpha,r,\cN$, \ac{HR} training adapts to and protects against whatever corruption plagues the data including combination of evasion and poisoning, while ensuring strong generalization. In all settings, \ac{HR} training achieves \ac{SOTA} performance and outperforms classical evasion defense mechanisms (\ac{AT} of \citet{madry2018towards}, TRADES of \citet{zhang2019theoretically}), poisoning defense mechanisms (DPA of \citet{levine2020deep,wang2022improved}), and the combination of poisoning and evasion defenses.
Furthermore, we replicate the experiments in \citet{rice2020overfitting} and provide numerical evidence that \ac{HR} training largely circumvents the robust overfitting phenomenon experienced by \ac{AT}.
Finally, we show empirically that the \textit{training loss} of \ac{HR} training is an upper bound on its \textit{adversarial test loss} with high probability, providing a generalization guarantee.
Notably, this observation does not hold for any of the prior robust approaches.

We release an open-source Python library (\url{https://github.com/RyanLucas3/HR_Neural_Networks}) that can be directly installed through \texttt{pip}. With our library, HR training can be used to provide robustness by changing essentially only one line of the training code. We also provide an extensive Colab tutorial on using HR in the github repository.

\subsection{Related work}\label{rel_work}

\textbf{Robust Overfitting.} Understanding and eliminating the robust overfitting phenomenon has been the topic of previous work.
\citet{rice2020overfitting} suggests that conventional remedies for overfitting cannot improve upon early stopping.
To compensate for poor generalization of \ac{AT} \cite{schmidt2018adversarially}, data augmentation techniques such as semi-supervised learning 
\cite{alayrac2019labels, carmon2019unlabeled} and data interpolation \cite{lee2020adversarial,chen2021guided} have been considered.
Alternative approaches directly modify the \ac{AT} loss function by reweighing data points
\cite{wang2019improving, zhang2021geometry}, mitigating memorization effects \cite{dong2021exploring} or favoring large loss data \cite{yu2022understanding}.
\citet{kulynych2022you} recently considers differential privacy as a way to ensure strong generalization.
Previous works mostly attempt to eliminate robust overfitting through empirical insight.
In our work, however, we present theoretical evidence with regards to the working of the robust overfitting mechanism (Theorem \ref{thm: AT gap}) and propose a disciplined approach at its elimination.
Our proposed \ac{HR} training is indeed a theoretically disciplined approach which comes with a theoretical certificate against overfitting (Theorem \ref{thm: upper-bound}).
On a practical level, \ac{HR} can be understood as a smart reweighing of the standard adversarial examples considered in \citet{madry2018towards} and in doing so equips standard \ac{AT} against with protection against both poisoning and evasion attacks. 

\textbf{Distributionally Robust Optimization.} \ac{WDRO} has recently received attention in the context of \ac{AT} \cite{staib2017distributionally,sinha2017certifying,wong2019wasserstein,bui2022unified}.
\ac{WDRO} safeguards against attackers with a bound on the average perturbation applied to all data points as opposed to standard \ac{AT} which protects against attacks in which the perturbation on each individual data point is bounded.
\ac{WDRO} however still suffers from fundamentally the same robust overfitting phenomenon as standard \ac{AT} as it does not take into account statistical error.
However, the \ac{KL} divergence ambiguity sets are known in the \ac{DRO} literature to be superior at providing safeguards against statistical error \cite{lam2019recovering, vanparys2021data} than their Wasserstein counterparts in the absence of data corruption.
The \ac{LP} metric has been shown in the statistical literature \cite{hampel1971general} to precisely captures certain types of data corruption.
Recently, \citet{bennouna2022holistic} proposed a novel ambiguity set combining both the \ac{KL} and \ac{LP} metrics in an attempt to safeguard against corruption and statistical error simultaneously.
Contrary to the \ac{WDRO} approach of \citet{sinha2017certifying}, the \acl{HR} \ac{DRO} formulation based on this novel ambiguity set is proven here to guarantee an upper bound on the adversarial test loss for \textit{any} desired corruption level and \textit{any} loss function.
Moreover, the \ac{HR} formulation is not harder to train than classical \ac{AT}.
This paper considers the ambiguity set of \citet{bennouna2022holistic} in the context of evasion (at \textit{test time}) and poisoning, and proves new statistical properties. In particular, we prove the HR ambiguity set to provide \textit{uniform} finite sample generalization bounds, independent of the dimentionality of the model class|a crucial novel property in the context of deep learning.

\section{Preliminaries: Learning under corruption}
We formalize here the setting of learning under evasion and poisoning corruption.
A typical underlying assumption in prior work is that the empirical distribution of the data is close to the out-of-sample distribution, neglecting statistical error and risking poor generalization.
Here, we carefully consider both types of corruption and their interaction with statistical error.

Consider a learning problem with covariates $x \in \cX$ and outputs $y \in \cY$ (e.g., labels in classification problems), following a joint distribution $\dist$ with compact support $\mathcal Z\subseteq \cX\times \cY$ which represents the distribution of the clean data. 
We do not observe samples directly from this clean distribution $\dist$ but rather up to a fraction $\alpha$ of these $n$ independent samples are wholly corrupted (poisoning corruption).
We hence observe the $n$ samples $\{z_i\defeq (x_i,y_i)\}_{i=1}^n \in \mathcal Z^n$ post corruption and denote with $\dist_n$ their empirical distribution.
During test time, the test data is again sampled from $\dist$ but is subjected to noise out of a set $\cN$ (evasion corruption).
Typical work in AT considers for example the noise bounded in an $l_p$ ball \cite{madry2018towards}. 
We denote the distribution of the perturbed test instances with $\dist_{\test}$.  
We observe $\dist_n$ but our goal is to learn a model that minimizes the test loss
\begin{equation*}
    \cost_{\test}(\theta) \defeq \E_{\dist_{\test}} [\loss(\theta,Z)],
\end{equation*}
where $\theta \in \Theta$ are the model parameters, $\ell$ is a bounded loss function (e.g., cross-entropy on neural networks); see Figure \ref{fig:data}.
Here, as in the remainder of the paper, $\E_{\dist'}$ denotes expectation over a random variable $Z$ with distribution $\dist'$.

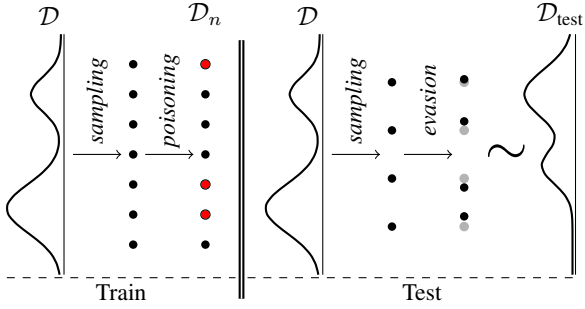
\begin{figure}
  \centering
  \begin{tikzpicture}[scale=0.8]


\def\xD{-4}
\draw (\xD-0.2,2) node[above] {$\dist$};
\draw[scale=1,thick,domain=-2:2,smooth,variable=\x]
  plot ({-0.5*exp(-6*(\x-1)^2) - 0.9*exp(-3*(\x+0.9)^2) + \xD},{\x});
\draw[-] (\xD+0.05,-2)--(\xD+0.05,2);
  
\def\ptsradius{2}

\def\xpts{-2.8}

\draw[->] (\xD+0.2,0)--(\xpts-0.2,0);
\draw (\xD+0.6,0.9) node[rotate=90] {\small \textit{sampling}};

\foreach \y in {-1.5, -1, -0.5, 0, 0.5,1,1.5}{
    \fill (\xpts,\y) circle[radius=\ptsradius pt];
}

\def\xptspoison{-1.6}

\draw[->] (\xpts+0.2,0)--(\xptspoison-0.2,0);
\draw (\xpts+0.6,0.9) node[rotate=90] {\small \textit{poisoning}};

\draw (\xptspoison,2) node[above] {${\dist}_n$};
\foreach \y in {-1.5, 0, 0.5,1}{
    \fill (\xptspoison,\y) circle[radius=\ptsradius pt];
}
\foreach \y in {-1, -0.5, 1.5}{
    \fill[red] (\xptspoison,\y) circle[radius=\ptsradius pt];
    \draw (\xptspoison,\y) circle[radius=2.2 pt];
}

\draw[thick,double] (-1,-2.4)--(-1,1.9);
\draw (-3,-2) node[below] {\small Train};
\draw[-,dashed] (-4.9,-2.05)--(-1.1,-2.05);
\draw (2,-2) node[below] {\small Test};
\draw[-,dashed] (-0.9,-2.05)--(4.55,-2.05);

\def\xDt{0.3}
\draw (\xDt-0.2,2) node[above] {$\dist$};

\draw[scale=1,thick,domain=-2:2,smooth,variable=\x]
  plot ({-0.5*exp(-6*(\x-1)^2) - 0.9*exp(-3*(\x+0.9)^2) + \xDt},{\x});
\draw[-] (\xDt+0.05,-2)--(\xDt+0.05,2);

\def\xptst{1.5}

\draw[->] (\xDt+0.2,0)--(\xptst-0.2,0);
\draw (\xDt+0.6,0.9) node[rotate=90] {\small \textit{sampling}};

\foreach \y in {-1.2, -0.4, 0.4, 1.2}{
    \fill (\xptst,\y) circle[radius=\ptsradius pt];
}

\def\xptstcor{2.7}

\draw[->] (\xptst+0.2,0)--(\xptstcor-0.2,0);
\draw (\xptst+0.6,0.9) node[rotate=90] {\small \textit{evasion}};

\foreach \y in {-1.2, -0.4, 0.4, 1.2}{
    \draw[draw=black!30, fill=black!30] (\xptstcor,\y) circle[radius=\ptsradius pt];
}

\fill (\xptstcor,-1.2+0.17) circle[radius=\ptsradius pt];
\fill (\xptstcor,-0.4-0.15) circle[radius=\ptsradius pt];
\fill (\xptstcor,0.4+0.15) circle[radius=\ptsradius pt];
\fill (\xptstcor,1.2+0.05) circle[radius=\ptsradius pt];

\def\xDtest{4.5}

\draw (\xptstcor+0.7,0) node {\huge $\sim$};

\draw (\xDtest-0.2,2) node[above] {$\dist_\test$};

\draw[scale=1,thick,domain=-2:2,smooth,variable=\x]
  plot ({-0.8*exp(-3*(\x-1)^2) - 0.5*exp(-6*(\x+0.2)^2) + \xDtest},{\x});
\draw[-] (\xDtest+0.05,-2)--(\xDtest+0.05,2);

\end{tikzpicture}
  \caption{Illustration of the distributions $\dist, \dist_n$ and $\dist_\test$. Here, the red points indicate the modified samples by poisoning.}
  \label{fig:data}
\end{figure}


\section{Adversarial training and generalization}
\label{sec: adv and generalization}

We now highlight the importance of robustness against statistical error when learning under corruption.
Consider the classical setting of adversarial evasion attacks in the absence of poisoning (i.e., $\alpha =0$) and let here $\loss^{\cN}(\theta, z) \defeq \max_{\delta \in \cN\!,\,  z+\delta\in \mathcal Z} \loss(\theta,z+\delta)$ when the adversary can perturb the data $z$ within the set $\cN \ni 0$.
The \ac{AT} models of \citet{madry2018towards} minimize the empirical adversarial loss
\(
\cost_\adv(\theta) \defeq \E_{\dist_n}\!\!\left[\loss^{\cN}(\theta, z) \right]
\)
and hope that its minimizer $\theta^\adv_{\dist_n}$ over the parameter set $\Theta$ is safeguarded against evasion attacks.
We have indeed that the testing loss $\cost_{\test}(\theta)$ is upper bounded by $\E_{\dist}\!\left[\loss^{\cN}(\theta, z) \right]$ and hence minimizing its empirical counterpart $\cost_\adv(\theta)$ makes intuitive sense.
However, due to statistical error there might be a considerable gap between the adversarial test loss $\E_{\dist}\!\left[\loss^{\cN}(\theta, z) \right]$ and the empirical adversarial loss $\cost_\adv(\theta)$ putting into question the intuitive appeal of this approach.
\ac{AT} typically does overfit to the worst-case perturbations of the \textit{data points}---rather than the out-of-sample distribution---and thus generalizes poorly.
This is akin to classical overfitting in \ac{ERM} caused by the gap between the minimized loss $\cost_\ERM(\theta) \defeq \E_{\dist_n} [\loss(\theta,Z)]$ and the out-of-sample cost $\E_{\dist} [\loss(\theta,Z)]$.
In the case of \ac{AT}, it has been empirically observed that overfitting is in fact exacerbated when compared to \ac{ERM} \cite{rice2020overfitting}.
We now provide theoretical insight into why \ac{AT} suffers from overfitting more than \ac{ERM}. 

First, we present a distributional perspective on \ac{AT} which will be useful in what follows.
Consider the ambiguity set 
\begin{equation}\label{eq:LP-set}
\mathcal{U}_{\cN}(\dist_n) \defeq \{ \dist_\test' : \LP_{\cN}(\dist_n, \dist'_\test) \leq 0 \}
\end{equation}
associated with the optimal transport metric 
\begin{align}    
&\LP_\cN(\dist_n, \dist_\test') 
\defeq \label{eq: LP-def}\\
&\inf \left\{ \int \mathbf{1}(z'-z \not \in \cN) \, \d\gamma(z,z') \; : \; \gamma \in \Gamma(\dist_n,\dist'_\test)\right\} \nonumber
\end{align}

where $\Gamma(\dist_n,\dist'_\test)$ denotes the set of all couplings on $\mathcal Z\times\mathcal Z$ between $\dist_n$ and $\dist'_\test$ \citep{villani2009optimal}.
It is straightforward to observe that
\begin{equation}\label{eq: AT DRO}
  \cost_\adv(\theta) = \max\{\E_{\dist'_\test}[\loss(\theta,Z)]\, : \dist'_\test \in \mathcal{U}_{\cN}(\dist_n)\}.
\end{equation}
Clearly, hence, naively considering \ac{DRO} formulations does not by itself eliminate robust overfitting.
We will show in Section \ref{sec:holistic-robust-loss} that a careful generalization of the previous \ac{DRO} formulation does protect against robust overfitting.


\paragraph{ERM overfitting gap.}
Let us first provide intuition on classical overfitting in \ac{ERM}.
As ${\dist}_n$ is here the empirical distribution of a random dataset it is itself a random variable and we denote with $\samples_n$ its distribution.
Denote the associated \ac{ERM} solution to some arbitrary distribution $\dist'$ as $\theta_{\dist'} \in \argmin_{\theta\in\Theta} \E_{\dist'}[\loss(\theta,Z)]$ which we assume to be unique for simplicity.
\ac{ERM} overfitting is bound to occur as the expected testing loss
is larger than the expected training loss,
i.e.,
\begin{align}
  \E_{\dist_n \sim \samples_n}\!\!\left[ 
                                    \E_{\dist}[\loss(\theta_{\dist_n}\!,\!Z)]
                                    \right] 
                                  \!\!=&
                                    \E_{\dist_n^1, \dist_n^2 \sim \samples_n}
                                    \![
                                    \E_{\dist_n^1}\![\loss(\theta_{\dist_n^2}\!,\!Z)]
                                    ] \label{eq: testing ERM reformulation} \\
                                  \geq&
                                    \E_{\dist_n^1, \dist_n^2 \sim \samples_n}
                                    \![
                                    \E_{\dist_n^1}\![\loss(\theta_{\dist_n^1}\!,\!Z)]
                                    ] \label{eq: training ERM reformulation} \\
                                  =&
                                    \E_{{\dist}_n \sim \samples_n}\!\!\left[ 
                                    \E_{{\dist}_n}[\loss({\theta}_{{\dist}_n},Z)]
                                    \right]. \label{eq: training ERM}
\end{align}
Here, Equality \eqref{eq: testing ERM reformulation} is due to Fubini's theorem.
Inequality \eqref{eq: training ERM reformulation} follows from the fact that $\theta_{\dist_n^1}$ is a minimizer of $\E_{\dist_n^1}[\loss(\theta,Z)]$.
Equality \eqref{eq: training ERM} follows from the fact that ${\dist_n^2}$ and ${\dist_n^1}$ are independent and share the same distribution $\samples_n$.
The right hand side of \eqref{eq: training ERM} is precisely the expected training loss.
We define the \ac{ERM} \textit{overfitting gap}, between training and testing loss, as 
\begin{align*}
  & \gap_\ERM(\samples_n,\loss) \\
  & \defeq \E_{\dist_n^1, \dist_n^2 \sim \samples_n}
            \left[
              \E_{\dist_n^2}[\loss(\theta_{{\dist}_n^1},Z)]
              -
              \E_{\dist_n^1}[\loss(\theta_{{\dist}_n^1},Z)]
              \right]\! \\
  &= \E_{\dist_n \sim \samples_n}\![\E_{\dist}[\loss(\theta_{\dist_n},\!Z)]\!-\!\E_{{\dist}_n}[\loss({\theta}_{{\dist}_n},\!Z)]].
\end{align*}
This gap definition is general in that it characterizes overfitting for any data generation process $\samples_n$.
It also highlights what fundamentally causes overfitting in \ac{ERM}: an expected difference in empirical distribution between training ($\dist_n^1$) and test ($\dist_n^2$) data.

\paragraph{AT overfitting gap.}
Let us now examine the robust overfitting phenomenon in \ac{AT}.
For an arbitrary distribution $\dist'$ and an arbitrary model $\theta$ we define its associated worst-case adversary as the distribution $\dist'(\theta) \in \arg\max\{\E_{\dist'_\test}[\loss(\theta,Z)]\, : \dist'_\test \in \mathcal{U}_{\cN}(\dist')\}$.
We denote in particular with $\dist'^\adv\defeq \dist'^\adv(\theta^\adv_{\dist'})$ the worst-case adversary of the \ac{AT} solution \(\theta^\adv_{\dist'}\in \arg
  \min_{\theta\in \Theta} \max\{\E_{\dist'_\test}[\loss(\theta,Z)]\, : \dist'_\test \in \mathcal{U}_{\cN}(\dist')\}
\) which minimizes $\cost_\adv$ (see \eqref{eq: AT DRO}).
We will assume in this section that the \ac{AT} best model $\theta^\adv_{\dist'}$ and its associated worst-case adversary $\dist'^\adv$ form a saddle-point.

\begin{assumption}[Saddle-Point]
  \label{assumption:saddle-point}
  For any $\dist'$ we have
  \(
  \E_{\dist'_\test}[\loss(\theta^\adv_{\dist'},Z)] \leq \E_{\dist'^\adv}[\loss(\theta,Z)] 
  \)
  for all $\theta\in \Theta$ and $\dist_{\test}'\in \mathcal U_{\cN}(\dist')$.
\end{assumption}

The previous saddle point condition implies in particular that the \ac{AT} best model also minimizes a nominal \ac{ERM} cost $\E_{\dist_n^\adv}[\loss(\theta,Z)]$ with respect to a worst-case adversary $\dist_n^\adv\sim \samples_n^\adv$, i.e.,
\(
  \textstyle\theta^\adv_{\dist_n} = \theta_{\dist^{\adv}_n} \in \arg\min_{\theta\in\Theta}\E_{\dist_n^\adv}[\loss(\theta,Z)].
\)
Assumption \ref{assumption:saddle-point} holds for any convex loss functions such as logistic regression \cite{sion1958general, bose2020adversarial}.
However, it is not necessarily verified for neural networks due to their inherent nonconvexity.
Recent work by \citet{gidel2021limited} suggests however that it may also hold for neural networks approximately.
We define the \textit{\ac{AT} overfitting gap} as  the expected difference between the adversarial testing loss and the \ac{AT} loss, i.e.,
\begin{align*}
& \gap_\adv(\samples_n, \loss)  \!\defeq\! \E_{\dist_n \sim \samples_n}\![\E_{\dist}[\loss^\cN\!(\theta^\adv_{\dist_n}\!,\!Z)]\!-\!\E_{{\dist}_n}[\loss^\cN\!(\theta^\adv_{{\dist}_n}\!,\!Z)]]\\
&\qquad =\E_{{\dist}_n \sim \samples_n} [ 
        \E_{\dist({\theta}^\adv_{{\dist}_n})}[\loss({\theta}^\adv_{{\dist}_n},Z)]
    -
        \E_{\dist_n^\adv}[\loss({\theta}^\adv_{{\dist}_n},Z)]
    ].
\end{align*}
Our next result shows that under a saddle-point assumption, the \ac{AT} overfitting gap can be decomposed into an \ac{ERM} overfitting gap and a \textit{nonnegative} term $\gap_\shift(\samples_n, \loss) \defeq  \E_{\dist_n^1,\dist_n^2 \sim \samples_n}[\E_{{\dist}_n^2({\theta}^\adv_{{\dist}_n^1})}[\loss(\theta^\adv_{\dist_n^1},Z)]-\E_{{\dist}_n^2({\theta}^\adv_{{\dist}_n^2})}[\loss(\theta^\adv_{\dist_n^1},Z)]]$.

\begin{theorem}\label{thm: AT gap}
  Let Assumption \ref{assumption:saddle-point} hold. Then, 
  \[
    \gap_\adv(\samples_n, \loss)= \gap_\ERM(\samples_n^\adv, \loss) 
      +
      \underbrace{\gap_\shift(\samples_n, \loss)}_{\geq 0}.
  \]
\end{theorem}

The proof of Theorem \ref{thm: AT gap} can be found in Appendix \ref{App: theory of overfitting} and indicates in particular that \ac{AT} overfitting gap is always at least as large as the overfitting gap suffered by an \ac{ERM} procedure on data generated by the worst-case adversary $\dist_n^\adv$.
Hence, while \ac{ERM} overfitting is due to difference between train and test data sets, \ac{AT} suffers worse overfitting as its gap is due to a difference between (adversarial) train and test datasets and an additional adversary shift term.
This result gives theoretical evidence for the robust overfitting phenomenon observed empirically in a flurry of recent papers.
The additional adversary shift term can be interpreted as the sensitivity of the \ac{AT} solution to a change of adversary. Hence, Theorem 3.2 indicates that robust overfitting is due to the trained model adapting too much to the specific artificially added perturbations during AT. 

Let us detail this adversary shift gap $\gap_\shift$ interpretation. Recall that ${\theta}^\adv_{{\dist}_n}$ is an AT model trained with distribution ${\dist}_n$. Intuitively, ${\dist}_n^2$ can be seen as a testing distribution, and ${\dist}_n^2({\theta}^\adv_{{\dist}_n^1})$ (resp. ${\dist}_n^2({\theta}^\adv_{{\dist}_n^2})$) is a testing distribution with adversarial examples where the adversary is chosen against model ${\theta}^\adv_{{\dist}_n^1}$ (resp. ${\theta}^\adv_{{\dist}_n^2}$). Hence, each of the two terms of the gap measures the robust testing error under a \textit{distinct} adversary. The difference measures therefore the gap of testing error when the adversary changes. This implies that the adversary shift quantifies the \textit{sensitivity of the AT model to a change of adversary} resulting from a change of training distribution: if we train with data ${\dist}_n^1$, AT trains with perturbation from adversary against ${\theta}^\adv_{{\dist}_n^1}$, while if we train with data ${\dist}_n^2$, AT trains with perturbation from a distinct (shifted) adversary, namely against ${\theta}^\adv_{{\dist}_n^2}$. This gap is large when the trained model adapts too much (overfits) to the specific added perturbation to the training samples with AT.

\section{A holistic robust loss}
\label{sec:holistic-robust-loss}

From the previous section, it is clear that practical robustness requires us to take into account statistical error in addition to poisoning and evasion corruption.
The key idea, following the motivation behind \ac{AT}, is to construct an upper bound on the test loss when learning with corruption.
As opposed to classical \ac{AT}, such upper bound should account for poisoning as well, and also, crucially, should bound the \textit{out-of-sample} adversarial test loss $\cost_\test$ rather than its empirical proxy.

Let us first build intuition on our upper bound. As Section \ref{sec: adv and generalization} demonstrate, an empirical proxy of the test loss typically underestimates the test loss of the trained model, as the expected overfitting gap is always positive $\gap_\ERM,\gap_\adv \geq 0$. This is caused by the fluctuations of the empirical distribution of samples around the out-of-sample distributions generating the samples. Hence, instead of considering the ambiguity set \eqref{eq:LP-set} of the \ac{AT} loss \eqref{eq: AT DRO} capturing only the corruption at hand, we augment this ambiguity set to capture the fluctuations of the empirical distribution due to sampling. These statistical fluctuations are captured by the \acl{KL} divergence \citep{vanerven2014renyi}, denoted here as KL, as we will demonstrate shortly. On the other hand, the optimal transport metric \ac{LP} (defined in \eqref{eq: LP-def}) will capture both evasion and poisoning. This augmentation will increase the estimated loss ``just enough'' to be a tight upper bound on the test loss.

Define here an ambiguity set
\(
\mathcal{U}_{\cN,\alpha,r}(\dist_n) \defeq \{ \dist_\test' : \exists \dist' ~\st~ \LP_{\cN}(\dist_n, \dist') \leq \alpha, \, \KL(\dist',\dist_\test') \leq r\}
\),
and the associated \acf{HR} loss as
\begin{equation}\label{eq: HR-loss}
  \cost_\HR ^{\cN,\alpha,r}(\theta) \!\defeq\!\max \{ \E_{\dist_\test'}[\loss(\theta,Z)] : \dist_\test'\! \in \mathcal{U}_{\cN,\alpha,r}(\dist_n) \},
\end{equation}
where the desired protection against evasion corruption is controlled by $\cN$, protection against poisoning corruption is controlled by $\alpha$, and statistical error is accounted for by $r$, respectively.
Intuitively, the LP ball $\{\mathcal{D}' \; : \; LP_{\mathcal{N}}(\mathcal{D}_n,\mathcal{D}') \leq \alpha\}$ contains all the possible ``training" empirical distributions $\mathcal{D}'$ directly sampled from the adversarial testing distribution $\mathcal{D}_{test}$|that is, by removing the poisoned data points and then perturbing each data point by evasion. As $\mathcal{D}'$ is an empirical distribution of samples from $\mathcal{D}_{test}$, then large deviation theory \cite{dembo2009large} ensures that the set $\{\mathcal{D}'_{test} \; : \; KL(\mathcal{D}',\mathcal{D}'_{test}) \leq r\}$ contains (intuitively) $\mathcal{D}_{test}$ with high probability $1-e^{-rn + O(1)}$. This intuitive explanation is merely a simplified illustration. In the proof of subsequent Theorem \ref{thm: upper-bound} (which we defer to Appendix \ref{app:holistic-robust-loss}) we prove indeed that the considered ambiguity set around the random observed empirical distribution contains the adversarial testing distribution (with probability $1-e^{-rn+O(1)}$), given that less than a fraction $\alpha$ of all training data points are corrupted by poisoning and the evasive attack on the test data is limited to a compact set in the interior of the set $\cN$.


\begin{theorem}[Robustness Certificate]\label{thm: upper-bound}
  Suppose that less than a fraction $\alpha\in (0,1]$ of the \ac{IID} training data is poisoned and the evasion attack on the test set is limited to a compact set in $\interior(\cN)$.
  Then,
  \(
    \Prob(\cost_\HR ^{\cN,\alpha,r}(\theta) \geq \cost_\test(\theta)~~\forall\theta\in\Theta)\geq 1-e^{-r n+O(1)}.
  \)
\end{theorem}

Denote here the \acl{HR} solution as $\theta^\HR_{\dist_n} \in \argmin_{\theta\in \Theta} \max_{\dist_\test'\in \mathcal{U}_{\cN,\alpha,r}(\dist_n) }  \E_{\dist_\test'}[\loss(\theta,Z)].$
The previous theorem certifies that asymptotically the testing error is not larger than the \ac{HR} loss with high probability, i.e.,
\[
  \Prob\!\left(\cost_\HR ^{\cN,\alpha,r}(\theta^\HR_{\dist_n}) \geq \cost_\test(\theta^\HR_{\dist_n}) \right)\geq 1-e^{-r n+O(1)}.
\]
In Appendix \ref{app:holistic-robust-loss} we prove a finite sample extension of this generalization certificate (Theorem \ref{theorem:main:finite}) and indicate that the results is statistically tight (Theorem \ref{theorem:main:tight}).
We point out that, unlike certificates founds in several prior works (for instance in \citet{sinha2017certifying}), our bound does not depend on the dimension of the parameter space $\Theta$ but rather only on the ``size'' of the event set $\mathcal Z$ (eg images and labels).
This is critical as in the context of neural networks typically the number of weights is taken much larger than the number of data points $n$.
The \ac{HR} loss is a natural generalization of the \ac{AT} loss. When no robustness against statistical error ($r=0$) and poisoning corruption ($\alpha=0$) is desired, we have indeed
\(
\cost_\HR ^{\cN,0,0}(\theta) = \cost_\adv(\theta);
\)
See Appendix \ref{sec:HR-appendix} for further details.


\section{HR training algorithm}

We now introduce an algorithm---HR training---to optimize efficiently the \ac{HR} loss $\cost_\HR ^{\cN,\alpha,r}$.
At first glance, minimizing the \ac{HR} loss seems challenging as it requires the solution of a saddle point problem over the model parameters $\theta\in \Theta$ and potentially continuous distribution $\dist_\test' \in \mathcal{U}_{\cN,\alpha,r}(\dist_n)$.
We exploit, however, a finite reformulation of the \ac{HR} loss by \citet{bennouna2022holistic}.
For any given training data (sub)set $\{z_i=(x_i,y_i)\}_{i\in I}$ with $I \subseteq [n]$, we can compute its associated \ac{HR} loss \textit{exactly} as 

\begin{align}
  \label{eq: HR-finite-formulation}
  \cost_\HR^{\cN,\alpha,r}(\theta) =  &
                                        \left\{\begin{array}{r@{~}l}
      \max & \sum_{i =1}^n d'_{i}\loss^{\cN}(\theta,z_i) + d'_{0} \loss^{\mathcal Z}(\theta) \\[0.5em]
      \st &  d'\in \Re^{n+1}_+, \, \hat d'\in \Re^{n+1}_+,\, s \in \Re^n_+,\\[0.5em]
           &  \sum_{i=0}^n d'_i = \sum_{i=0}^n \hat{d}'_i = 1,\\[0.5em]
           &\hat{d}'_i \!+\! s_i\!=\!\frac{1}{n}~~\forall i \!\in\! [1,\dots,n],\\[0.5em]
           & \sum_{i=0}^n \hat{d}'_i \log ( \frac{\hat{d}_i}{d'_i} ) \leq r,~\sum_{i=1}^n s_i\leq \alpha
    \end{array}\right.
\end{align}


where $\loss^{\mathcal Z}(\theta)\defeq \max_{z\in\mathcal Z} \ell(\theta, z)$ and $\loss^{\cN}(\theta, z) \defeq \max_{\delta \in \cN\!,\,  z+\delta\in \mathcal Z} \loss(\theta,z+\delta)$.
In particular, the supremum of the HR loss \eqref{eq: HR-loss} is attained in distributions $\mathcal{D}'_{test}$ and $\mathcal{D}'$ of finite support ($d'$ and $\hat{d}'$ respectively in \eqref{eq: HR-finite-formulation}), specifically the adversarial examples and a point maximizing the loss.
The \ac{HR} loss is therefore in essence simply an adversarial reweighing of loss terms comprising the \ac{AT} loss and an additional loss term $\loss^{\mathcal Z}(\theta)$.
Determining the \ac{HR} loss exactly requires evaluating the adversarial loss $\loss^{\cN}(\theta, z)$ and 
subsequently 
solving the exponential cone problem \eqref{eq: HR-finite-formulation} with $O(n)$ variables and constraints, which can be done efficiently by off-the-shelf optimization solvers.


Practically, we train the \ac{HR} loss by mimicking standard \acl{AT}.
At each minibatch iteration, we indeed first compute an approximate adversarial example $z'_i \approx z_i + \argmax_{\delta \in \cN, z_i+\delta\in \mathcal Z} \loss(\theta,z_i+\delta)$ for each data point in the minibatch $I \subset [n]$ using the \ac{PGD} algorithm of \citet{madry2018towards}.
With the help of the approximations $\loss^{\cN}(\theta, z_i)\approx \loss(\theta,z'_i)$ as well as $\loss^{\mathcal Z}(\theta) \approx \max_{i\in I} \loss(\theta,z'_i)$ we determine an approximate maximizer $d'^\star$ in problem \eqref{eq: HR-finite-formulation}.
The model parameters are then updated by stepping in the direction of the negative gradient of the \ac{HR} loss. We remark that this gradient can be computed efficiently using Danskin's theorem as a weighted sum of the gradients on the adversarial examples, i.e., $\nabla_\theta \cost_\HR^{\cN,\alpha,r}(\theta;I) = \sum_{i\in I} d'^\star_i\nabla_{\theta}\loss(\theta,z_i') + d'^\star_{0} \nabla_{\theta}\loss(\theta, z'_{\arg \max_{i\in I}\! \loss(\theta,z'_i)})$.
The algorithm is summarized in Algorithm \ref{alg:HR-training} where we remark that \ac{HR} training is different from standard \acl{AT} only in that it requires computing the maximizer $d'^\star$ in \eqref{eq: HR-finite-formulation}.
Practically, \ac{HR} training with \ac{PGD} takes less than 6\% more time than standard \acl{AT} using \ac{PGD} in typical benchmarks.

\begin{algorithm}
 \caption{Holistic Robust Training}
  \textbf{Specification:} Learning rate $\lambda$, number of epochs $T$, mini-batches $B$, minibatch replays $M$, HR parameters ($\cN, \alpha, r)$.
    \textbf{Input:} Data $\{z_i=(x_i, y_i)\}_{i\in[n]}$.
  Initialized $\theta$. \\ 
  {\color{black!50} //iterate $M$ times per batch, for a total of $T$ epochs}\\
  \For{$t\in [T/M]$, ~minibatch $I \in B$,~$\_\in[M]$}{
    
    {\color{black!50} //adversarial attack, e.g., using \acs{PGD}}\\
    Compute $z'_i \!\approx\! z_i+{\argmax_{\delta \in \cN}}\, \loss(\theta, z_i\!+\!\delta)\,\forall i \in I$,\\ 
    {\color{black!50} //reweigh data points}\\
    Compute optimal weights $d'^\star$ in problem \eqref{eq: HR-finite-formulation} ,\\
    {\color{black!50} //update model with gradient descent}\\
    $\theta \leftarrow \theta - \lambda\nabla_{\theta} \cost_\HR^{\cN,\alpha,r}(\theta; I)$
  }
  \textbf{Output:} Model parameters $\theta$ after $T$ epochs.
\label{alg:HR-training}
\end{algorithm}

\section{Experiments}
\label{sec: experiments}

In this section, we investigate the efficacy of \ac{HR} training through extensive experiments on the {MNIST} and {CIFAR-10} datasets\footnote{The code reproducing all experiments is available at: \url{https://github.com/RyanLucas3/HR_Neural_Networks}}.
We consider the crossentropy loss $\ell(\theta, (x,y))=-\sum_{k=1}^K y^{(k)} \log f^{(k)}_{\theta}(x)$, $K=10$, where $f_{\theta}(x)\in\Re^K$ represents the network with weights $\theta$ and softmax output probabilities for covariate $x$ and one-hot encoded class label $y\in\Re^K$.
We present here only the {CIFAR-10} results and leave the {MNIST} experiments, which bear the same insights, to Appendix \ref{app:experiments}.  
We conduct three sets of experiments.
The first set evaluates each type of robustness separately and shows that each parameter provides robustness against a distinct source of overfitting: $\cN$ against data evasion, $\alpha$ against data poisoning and finally $r$ against statistical error.
The second set of experiments shows that HR training does not suffer robust overfitting and also benchmarks \ac{HR} training against \ac{SOTA} evasion robustness methods.
The last set of experiments evaluates robustness against statistical error and both corruptions simultaneously and shows that \ac{HR} training significantly outperforms benchmarks.
In particular, each of these sets of experiments shows the theoretical robustness certificate provided in Theorem \ref{thm: upper-bound} holds empirically.

In these experiments, we train \ac{HR} with \ac{PGD} (see Algorithm \ref{alg:HR-training}) and adopt the classical setting of covariates evasion attacks by choosing $\cN = {B}(0,\epsilon)\times \{0\}$ where depending on the context ${B}(0,\epsilon)$ is here either a Euclidean or infinity norm ball with radius $\epsilon$.

\subsection{Evaluating Robustness}
\label{sec:Evaluating Robustness}

In this set of experiments, we use a ResNet18 architecture for CIFAR-10.
We detail the used
architectures and network characteristics in Appendix \ref{App:
  experimental_setup}.

\textbf{Robustness to Statistical Error $(\epsilon = 0, \alpha = 0, r \geq 0)$.}\label{RSE}
To study the robustness of \ac{HR} to statistical error, we train \ac{HR} with various parameters $r$ and growing data size---associated with decreasing levels of statistical error.
We limit training data to $5\%, 10\%, 25\%$ and $50\%$ of all available training data.
We run 10 trials of \ac{HR} training for each parameter value $r \in \{0, 0.05, 0.1,  0.25\}$ on such random subsampled training data. Figure \ref{fig:single_param} depicts the average and standard deviation of the testing loss across all trials.
The parameter $r$ (reflecting the use of KL distributional robustness) clearly improves performance, and its effect is considerably stronger with lower data sizes which are associated with more significant statistical error. Moreover, variance of the loss with respect to the random subsamples in each trail decreases with larger $r$. These observations are in line with the theory that \ac{KL} distributional robustness provides efficient robustness against statistical error \citep{gotoh2018robust}.
In Appendix \ref{App: single_param} we present more extensive plots. 

\begin{figure}
\centering
\begin{subfigure}
  \centering
  \scalebox{.48}{\begin{tikzpicture}

\definecolor{cornflowerblue}{RGB}{100,149,237}
\definecolor{darkgray176}{RGB}{176,176,176}
\definecolor{darkorange}{RGB}{255,140,0}
\definecolor{green}{RGB}{0,128,0}
\definecolor{lightcoral}{RGB}{240,128,128}
\definecolor{mediumseagreen}{RGB}{60,179,113}

\begin{axis}[
legend cell align={left},
legend style={fill opacity=0.8, draw opacity=1, text opacity=1},
tick align=outside,
tick pos=left,
title={Robustness to statistical error},
x grid style={darkgray176},
xlabel={Data Size},
xmajorgrids,
xmin=-0.15, xmax=3.15,
xtick style={color=black},
xtick={0,1,2,3},
xtick={0,1,2,3},
xtick={0,1,2,3},
xtick={0,1,2,3},
xtick={0,1,2,3},
xticklabels={5\%,10\%,25\%,50\%},
xticklabels={5\%,10\%,25\%,50\%},
xticklabels={5\%,10\%,25\%,50\%},
xticklabels={5\%,10\%,25\%,50\%},
xticklabels={5\%,10\%,25\%,50\%},
y grid style={darkgray176},
ylabel={Test Loss},
ymajorgrids,
ymin=0.227485875936704, ymax=4.2488957281123,
ytick style={color=black},
ytick={0,0.5,1,1.5,2,2.5,3,3.5,4,4.5},
yticklabels={
  \(\displaystyle {0.0}\),
  \(\displaystyle {0.5}\),
  \(\displaystyle {1.0}\),
  \(\displaystyle {1.5}\),
  \(\displaystyle {2.0}\),
  \(\displaystyle {2.5}\),
  \(\displaystyle {3.0}\),
  \(\displaystyle {3.5}\),
  \(\displaystyle {4.0}\),
  \(\displaystyle {4.5}\)
}
]
\path [draw=red, fill=red, opacity=0.4]
(axis cs:0,2.3703170959342)
--(axis cs:0,2.16221748658534)
--(axis cs:1,1.53071819312287)
--(axis cs:2,0.878886857520971)
--(axis cs:3,0.580555455191598)
--(axis cs:3,0.686954413735404)
--(axis cs:3,0.686954413735404)
--(axis cs:2,1.00886874459084)
--(axis cs:1,1.69101733956146)
--(axis cs:0,2.3703170959342)
--cycle;

\path [draw=green, fill=green, opacity=0.4]
(axis cs:0,2.19184004668349)
--(axis cs:0,1.84270197495347)
--(axis cs:1,1.35197395792523)
--(axis cs:2,0.753170320597491)
--(axis cs:3,0.506421469236509)
--(axis cs:3,0.694295439542635)
--(axis cs:3,0.694295439542635)
--(axis cs:2,0.909244413823285)
--(axis cs:1,1.539165706534)
--(axis cs:0,2.19184004668349)
--cycle;

\path [draw=darkorange, fill=darkorange, opacity=0.4]
(axis cs:0,1.77958950159278)
--(axis cs:0,1.49223567022119)
--(axis cs:1,1.13159347998093)
--(axis cs:2,0.616548670836025)
--(axis cs:3,0.410277232853777)
--(axis cs:3,0.596136137049788)
--(axis cs:3,0.596136137049788)
--(axis cs:2,0.720498548974461)
--(axis cs:1,1.25069632104446)
--(axis cs:0,1.77958950159278)
--cycle;

\path [draw=cornflowerblue, fill=cornflowerblue, opacity=0.4]
(axis cs:0,4.06610437119523)
--(axis cs:0,3.39317776228865)
--(axis cs:1,2.31367603844091)
--(axis cs:2,1.39113732211158)
--(axis cs:3,1.02403629811486)
--(axis cs:3,1.16856177454749)
--(axis cs:3,1.16856177454749)
--(axis cs:2,1.63710775776818)
--(axis cs:1,2.7540183641775)
--(axis cs:0,4.06610437119523)
--cycle;

\addplot [ultra thick, cornflowerblue, mark=*, mark size=2, mark options={solid}]
table {%
0 3.72964106674194
1 2.5338472013092
2 1.51412253993988
3 1.09629903633118
};
\addlegendentry{$r = 0$ (ERM)}
\addplot [ultra thick, lightcoral, mark=*, mark size=2, mark options={solid}]
table {%
0 2.26626729125977
1 1.61086776634216
2 0.943877801055908
3 0.633754934463501
};
\addlegendentry{$r = 0.05$}
\addplot [ultra thick, mediumseagreen, mark=*, mark size=2, mark options={solid}]
table {%
0 2.01727101081848
1 1.44556983222961
2 0.831207367210388
3 0.600358454389572
};
\addlegendentry{$r = 0.1$}
\addplot [ultra thick, darkorange, mark=*, mark size=2, mark options={solid}]
table {%
0 1.63591258590698
1 1.1911449005127
2 0.668523609905243
3 0.503206684951782
};
\addlegendentry{$r = 0.2$}
\end{axis}

\end{tikzpicture}

}
\end{subfigure}%
\begin{subfigure}
  \centering
  \scalebox{.48}{\begin{tikzpicture}

\definecolor{cornflowerblue}{RGB}{100,149,237}
\definecolor{darkgray176}{RGB}{176,176,176}
\definecolor{darkorange}{RGB}{255,140,0}
\definecolor{lightcoral}{RGB}{240,128,128}
\definecolor{mediumseagreen}{RGB}{60,179,113}

\begin{axis}[
legend cell align={left},
legend style={fill opacity=0.8, draw opacity=1, text opacity=1},
tick align=outside,
tick pos=left,
title={Robustness to data poisoning},
x grid style={darkgray176},
xlabel={Misspecified Labels},
xmajorgrids,
xmin=-0.1, xmax=2.1,
xtick style={color=black},
xtick={0,1,2},
xtick={0,1,2},
xtick={0,1,2},
xtick={0,1,2},
xtick={0,1,2},
xticklabels={25\%,10\%,5\%},
xticklabels={25\%,10\%,5\%},
xticklabels={25\%,10\%,5\%},
xticklabels={25\%,10\%,5\%},
xticklabels={25\%,10\%,5\%},
y grid style={darkgray176},
ylabel={Test Loss},
ymajorgrids,
ymin=0.396381171354614, ymax=1.29013965828573,
ytick style={color=black},
ytick={0.3,0.4,0.5,0.6,0.7,0.8,0.9,1,1.1,1.2,1.3},
yticklabels={
  \(\displaystyle {0.3}\),
  \(\displaystyle {0.4}\),
  \(\displaystyle {0.5}\),
  \(\displaystyle {0.6}\),
  \(\displaystyle {0.7}\),
  \(\displaystyle {0.8}\),
  \(\displaystyle {0.9}\),
  \(\displaystyle {1.0}\),
  \(\displaystyle {1.1}\),
  \(\displaystyle {1.2}\),
  \(\displaystyle {1.3}\)
}
]
\path [draw=cornflowerblue, fill=cornflowerblue, opacity=0.3]
(axis cs:0,1.21027701906799)
--(axis cs:0,1.10346971509339)
--(axis cs:1,0.763403737676062)
--(axis cs:2,0.63670330951037)
--(axis cs:2,0.740590734602789)
--(axis cs:2,0.740590734602789)
--(axis cs:1,0.843300714380823)
--(axis cs:0,1.21027701906799)
--cycle;

\path [draw=lightcoral, fill=lightcoral, opacity=0.3]
(axis cs:0,0.949472398143599)
--(axis cs:0,0.875513246845415)
--(axis cs:1,0.553594234142281)
--(axis cs:2,0.474850936058904)
--(axis cs:2,0.506180474997615)
--(axis cs:2,0.506180474997615)
--(axis cs:1,0.625505540637992)
--(axis cs:0,0.949472398143599)
--cycle;

\path [draw=mediumseagreen, fill=mediumseagreen, opacity=0.3]
(axis cs:0,0.940946012775083)
--(axis cs:0,0.846832828625063)
--(axis cs:1,0.550631192824535)
--(axis cs:2,0.43700655712421)
--(axis cs:2,0.46884377956097)
--(axis cs:2,0.46884377956097)
--(axis cs:1,0.609229824631519)
--(axis cs:0,0.940946012775083)
--cycle;

\path [draw=darkorange, fill=darkorange, opacity=0.3]
(axis cs:0,1.24951427251613)
--(axis cs:0,0.856460672247053)
--(axis cs:1,0.597805876172632)
--(axis cs:2,0.480436497741837)
--(axis cs:2,0.512588231567245)
--(axis cs:2,0.512588231567245)
--(axis cs:1,0.667996573473364)
--(axis cs:0,1.24951427251613)
--cycle;

\addplot [ultra thick, cornflowerblue, mark=*, mark size=2, mark options={solid}]
table {%
0 1.15687336708069
1 0.803352226028442
2 0.68864702205658
};
\addlegendentry{$\alpha = 0$ (ERM)}
\addplot [ultra thick, lightcoral, mark=*, mark size=2, mark options={solid}]
table {%
0 0.912492822494507
1 0.589549887390137
2 0.490515705528259
};
\addlegendentry{$\alpha = 0.05$}
\addplot [ultra thick, mediumseagreen, mark=*, mark size=2, mark options={solid}]
table {%
0 0.893889420700073
1 0.579930508728027
2 0.45292516834259
};
\addlegendentry{$\alpha = 0.1$}
\addplot [ultra thick, darkorange, mark=*, mark size=2, mark options={solid}]
table {%
0 1.05298747238159
1 0.632901224822998
2 0.496512364654541
};
\addlegendentry{$\alpha = 0.2$}
\end{axis}

\end{tikzpicture}
}
\end{subfigure}%
\caption{Robustness to statistical error / poisoning on CIFAR-10. }
\label{fig:single_param}
\end{figure}
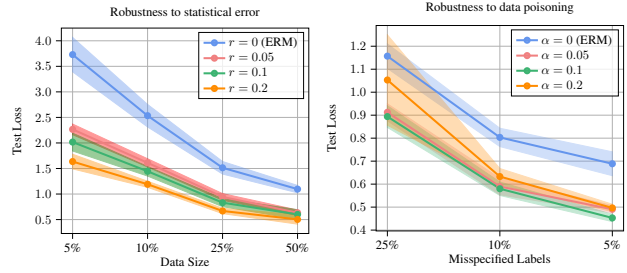


\textbf{Robustness to Poisoning $(\epsilon = 0, \alpha \geq 0, r = 0)$.}\label{RM}
To study HR robustness to poisoning, we evaluate HR training under label-flipping attacks.
We select $0\%, 5\%, 10\%, 25\%$ of training images randomly and flip their labels uniformly at random to a different label.
We then run 10 trials of HR training for each parameter value $\alpha \in \{0, 0.05, 0.1, 0.2\}$ on the resulting randomly corrupted data. Figure \ref{fig:single_param} present the average and standard deviation of the testing loss across all trials. HR with parameter $\alpha$ clearly improves performance when learning under poisoning corruption and its improvement is more significant with larger corruption levels. When $\alpha$ is chosen too large, the learned model is naturally overly conservative which decreases performance. We present further detailed plots of the results in Appendix \ref{App: single_param}, showing that the learned model indeed provides a robustness certificate as the HR loss is an upper bound on the testing loss whenever $\alpha$ is chosen roughly larger or equal to the corruption level.
\vspace{4em}

\textbf{Robustness to Evasion $(\epsilon \geq 0, \alpha = 0, r=0)$}\label{RN}
When $\alpha = 0$ and $r=0$, the HR loss becomes the classical AT loss, and HR training with PGD (Algorithm \ref{alg:HR-training}) becomes the PGD algorithm of \citet{madry2018towards} which has been well documented to provide robustness against evasion attacks.

\subsection{Combatting robust overfitting }\label{rob_of}

We replicate here the experiment of \citet{rice2020overfitting} exhibiting robust overfitting for \ac{AT} training.
As in \citet{rice2020overfitting} we use preactivation ResNet18 on CIFAR-10 with PGD attacks of 10 steps, and with SGD learning rate decay at epochs 100 and 150. Further experimental setup is discussed in Appendix \ref{App: experimental_setup}. We train \ac{AT} with PGD of \citet{madry2018towards}, TRADES of \citet{zhang2019theoretically}, and HR with $r\geq 0$, $\epsilon >0$ and $\alpha=0$. For all algorithms, we use the same PGD attack on training with $\epsilon= 8/255$ and 10 attack steps.
Figure \ref{fig:rob_overfitting} shows the test loss and error up to 200 training epochs when using {pre-activation ResNet18 under $\ell_{\infty}$ attacks}. 
Table \ref{tab: rice_tab} reports the test error at the 200 epoch mark, replicating the results reported in Table 2 in \citet{rice2020overfitting}. \ac{AT} clearly exhibits robust overfitting as the test loss drastically increases with more training epochs. The test error also increases but perhaps not as drastically. This indicates that although the number of errors made eventually settles, worryingly, the errors are made with increasing confidence as quantified by the crossentropy loss.
Similar to \citet{rice2020overfitting}, we find that TRADES \cite{zhang2019theoretically} without early stopping also experiences robust overfitting although to a lesser extent than regular PGD training.
\ac{HR} training with modest parameter values $r$ effectively mitigates this phenomenon and results in a monotonically decreasing test loss as a function of training epochs eliminating the need for early stopping.
It is worth noting that PGD best checkpoint performs better than HR final checkpoint. However, it might be practically hard to recover the performance of such best checkpoint with early stopping. In fact, PGD's training curve around the best checkpoint is very steep, and hence, early stopping might result in unstable performance on practical datasets, with potentially lower datasize.

\begin{figure}[h]
\centering
\scalebox{.52}{\input{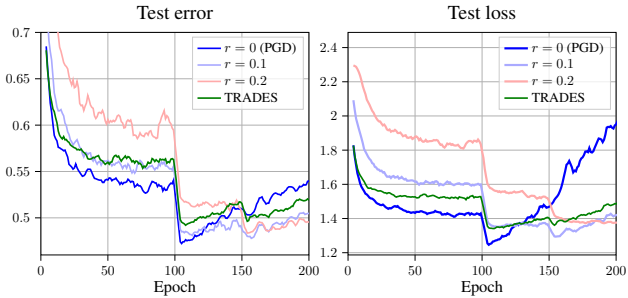}}\caption{Standard AT training versus HR training test loss for evasion with $\epsilon = 8/255$ on CIFAR-10. Robustness to statistical error ($r>0$) reduces robust overfitting experienced by \ac{AT}.}\label{fig:rob_overfitting}
\end{figure}

\begin{table}[h]
\smaller
\begin{tabular}{lllrr}
\toprule
\multicolumn{2}{c}{ \textbf{ROBUST TEST ERROR (\%) }}  &    FINAL &   BEST &   DIFF \\
\midrule
& E. Stopng. &  46.9 & 46.7 &  0.2 \\
\midrule
\textbf{Robust} & \text{plain-PGD}             &  53.9 & 46.8 &  7.2 \\
\textbf{Methods} & PGD+$\ell_1$-Reg.            &  53.0  &  \textbf{46.4} &  4.4 \\
& PGD+$\ell_2$-Reg.               &  55.2  &  48.6 &  8.8 \\
& TRADES            &  52.0  &  48.8 &  3.2 \\
& HR (Ours)             &  \textbf{49.6}  &  47.8 &  \textbf{1.8} \\
\bottomrule
\end{tabular}

	\caption{Replicating Table 2 from \cite{rice2020overfitting}. For each model, the optimal hyperparameters are selected as in \citet{rice2020overfitting}. }
	\label{tab: rice_tab}
\end{table}

\begin{figure}[h]
  \centering
  \includegraphics[width=0.49\textwidth]{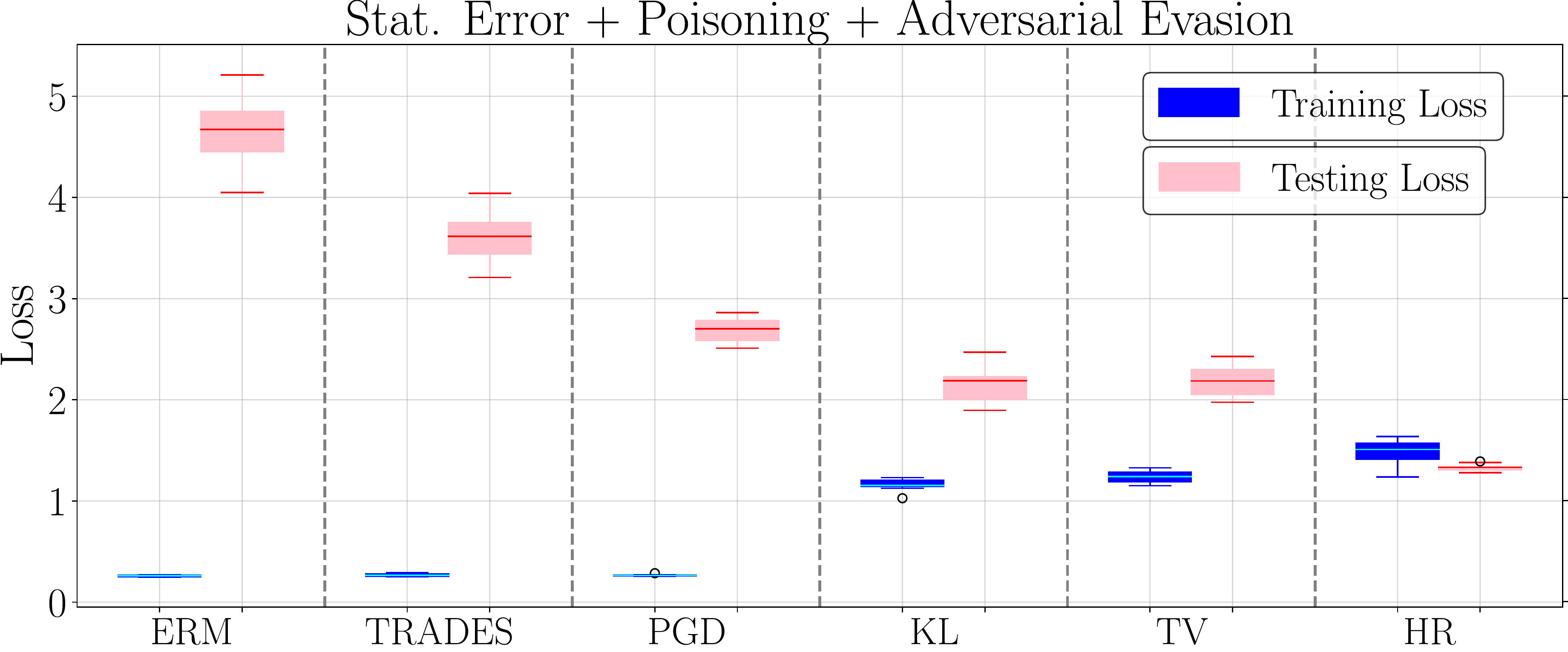}
  \caption{Training and adversarial testing loss of HR and various benchmarks in the presence of statistical error, data poisoning and evasion attacks. Training loss refers to the loss minimized during training while testing loss refers to the loss on the adversarially perturbed test set.}
  \label{fig: validation}
\end{figure}



\begin{table*}[h!]
\smaller
\centering
\begin{tabular}{lcccccc}
\toprule
& {} & \parbox[c]{1.5cm}{Natural \\ (No Attack)} & \parbox[c]{1.1cm}{Evasion} & \parbox[c]{1.6cm}{Stat.Error + Poisoning} & \parbox[c]{1.5cm}{Stat.Error + Evasion} & \parbox[c]{1.6cm}{Stat.Error + Poisoning +  Evasion} \\
\midrule
\textbf{Combined Defenses} & HR        & \textbf{9.0\%}   & \textbf{18.5\%}  & \textbf{25.3\%} (\textcolor{purple}{\textbf{27.1\%}}) & \textbf{33.0\%} (\textcolor{purple}{\textbf{34.4\%}}) & \textbf{36.6\%} (\textcolor{purple}{\textbf{38.9\%}}) \\
& PGD + DPA & 18.6\%            & 27.8\%            & 37.0\% (\textcolor{purple}{38.2\%})                      & 47.1\% (\textcolor{purple}{48.0\%})                        & 45.2\% (\textcolor{purple}{46.1\%}) \\
\midrule
\textbf{AT Defenses} & PGD       & 10.6\%            & 20.2\%            & 30.1\% (\textcolor{purple}{32.4\%})                      & 33.8\% (\textcolor{purple}{35.3\%})                        & 40.6\% (\textcolor{purple}{41.7\%}) \\
& TRADES    & 13.9\%            & 32.8\%            & 32.6\% (\textcolor{purple}{34.3\%})                      & 43.3\% (\textcolor{purple}{46.9\%})                        & 49.7\% (\textcolor{purple}{52.0\%}) \\
\midrule
\textbf{Poisoning Defenses} & DPA       & 15.6\%            & 39.2\%            & 35.4\% (\textcolor{purple}{37.9\%})                      & 50.3\% (\textcolor{purple}{51.4\%})                        & 49.6\% (\textcolor{purple}{51.4\%}) \\
& TV        & 9.8\%             & 40.8\%            & \textbf{25.3\%} (\textcolor{purple}{\textbf{27.1\%}}) & 43.9\% (\textcolor{purple}{50.5\%}) & 50.6\% (\textcolor{purple}{54.4\%}) \\
\midrule
\textbf{No Defense} & ERM       & 9.4\%             & 35.4\%            & 28.8\% (\textcolor{purple}{30.5\%})                      & 43.9\% (\textcolor{purple}{50.5\%})                        & 52.2\% (\textcolor{purple}{55.5\%}) \\
& KL        & \textbf{9.0\%}   & 38.1\%            & 28.4\% (\textcolor{purple}{31.5\%})                      & 43.9\% (\textcolor{purple}{50.5\%})                        & 49.9\% (\textcolor{purple}{54.7\%}) \\
\bottomrule
\end{tabular}
\caption{Mean and (\textcolor{purple}{worst-case}) error rate over 10 trials with model parameters chosen based on a validation set.
}
\label{tab:HR_results_full}
\end{table*}

\subsection{Holistic Robustness }\label{HR}
We finally compare the robustness of HR to benchmarks on datasets in all possible corruption settings: natural data (no corruption), evasion corruption, evasion corruption with subsampling, poisoning corruption with subsampling, simultaneous evasion and poisoning corruption with subsampling. As in Section \ref{sec:Evaluating Robustness}, subsampling allows to measure the effect of statistical error.
For each setting, we test the performance of each algorithm on 10 trials of randomly sampled and corrupted data.
In each trial, when statistical error is considered, we train algorithms with a random 25\% subsample of the full dataset. When poisoning attacks are considered, we further randomly flip  10\% of the training labels. The test loss performance is then evaluated with the help of a separate test set. When evasion is considered, the test set is adversarially perturbed by PGD-10 attacks with an $\ell_2$ adversary of magnitude $\epsilon = 0.1$.




We benchmark HR with the following algorithms: (i) algorithms with no built-in robustness: ERM, (ii) algorithms designed to protect against evasion attacks: \ac{AT} with PGD \cite{madry2018towards} and TRADES \cite{zhang2019theoretically} (iii) algorithms designed to protect against statistical error: \ac{KL} DRO \cite{namkoong2017variance}; a particular case of HR with $\alpha=0, \epsilon=0$ (see Appendix \ref{sec:HR-appendix}), (iv) algorithms designed to protect against poisoning corruption: DPA \cite{wang2022improved} and \ac{TV} DRO; particular case of HR with $r=0, \epsilon=0$ (see again Appendix \ref{sec:HR-appendix}), (v) Combining evasion and poisoning defenses algorithms: combination of PGD and DPA; see Appendix \ref{app:benchmarking} for details.
The hyperparameters of each algorithm are selected based on a 70/30 train/validation split of the training data as detailed in Appendix \ref{app:benchmarking}. 

\textbf{Results.} 
Table \ref{tab:HR_results_full} presents the natural and adversarial classification errors of each algorithm under various corruption settings. HR training significantly surpasses all other benchmarks in \textit{every setting}. Traditional robust algorithms, such as PGD, TRADES, and DPA, unsurprisingly falter when faced with a corruption variant different from what they were originally designed to handle, performing even worse than ERM. In stark contrast, HR training demonstrates its adaptability by automatically adjusting to the specific corruption in the data, consistently maintaining high performance even when faced with combined attack strategies. 
Furthermore, the superior generalization capability of HR enables it to outperform each algorithm even in the unique attack settings for which they were specifically designed. 
Finally, HR also significantly outperforms the combination of defense techniques (PGD+DPA). Indeed, combining these defense techniques does not necessarily yield cumulative benefits due to the potential negative interactions between them. In contrast, HR is designed to provide effective simultaneous defenses while ensuring strong generalization.

Figure \ref{fig: validation} shows the training loss and the adversarial test loss of each algorithm
in the settings of combined evasion and poisoning. Here, DPA and DPA+PGD are not represented a they have no clear notion of loss. HR training significantly outperforms all other benchmarks in two distinct ways.
First, \ac{HR} training enjoys the smallest adversarial test loss. Furthermore, the variance of its adversarial test loss is the smallest among the benchmarked procedures.
Second, \ac{HR} training effectively provides a robustness certificate as alluded to in Theorem \ref{thm: upper-bound}. The training loss is indeed with high probability an upper bound on the testing loss---an observation which fails to hold for the other procedures.


\paragraph{Training time.}
We investigate the train time of HR to evaluate the computational cost of its gain in robustness and generalization.
\autoref{tab:runtime} reports the training time (in minutes) per epoch of ERM, PGD, and HR on various datasets and architectures.
While HR takes four times longer than PGD for small datasets and architectures (ConvNet on MNSIT), its additional computational cost compared to PGD is negligible for large datasets and architectures: \textbf{$<$+6\%} for CIFAR-10 with ResNet and TinyImageNet with ResNet/EfficientNet. This is due to the fact that computationally, HR training only differ with PGD in solving an additional conic optimization problem \eqref{eq: HR-finite-formulation} as described in Algorithm \ref{alg:HR-training}. 
This optimization problem depends only on the batch size (linearly) and is \textit{independent} on the data dimensionality and the network's size/architecture. As a result, while the PGD step and the network gradient step scale with the network's size and data dimensionality, HR's additional step does not scale with these factors, keeping its additional computational burden constant compared to PGD.

\begin{table}[h!]
\centering
\begin{tabular}{@{}lllll@{}}
\toprule
Architecture         & Dataset       & ERM   & PGD   & HR    \\ \midrule
ConvNet       & MNIST         & 0.04  & 0.28  & 1.30  \\
ResNet-18     & CIFAR-10      & 0.24  & 2.60  & 2.75  \\
ResNet-18     & Tiny ImageNet & 3.39  & 6.06  & 6.36  \\ 
EfficientNet  & Tiny ImageNet & 7.65  & 39.59 & 40.87 \\ \bottomrule
\end{tabular}
\caption{Runtime of ERM, PGD, and HR in minutes per epoch (one full pass of the training dataset).}
\label{tab:runtime}
\end{table}



\section{Discussion \& Conclusion}

We make neural network training resilient against adversarial attacks based on a disciplined robust optimization approach. We remark that our approach can be generalized to other corruption models by substituting the \ac{LP} metric with a metric suitable for the considered corruption model (e.g., Wasserstein for perturbations bounded on average). The tractability properties of the resulting method may however depend on the considered metric. We believe hence that the considered approach constitutes an interesting new inroad to bridging generalization and adversarial robustness of modern deep neural networks.


\bibliography{references}
\bibliographystyle{icml2022}

\newpage
\appendix
\onecolumn







\section{Theory of overfitting}\label{App: theory of overfitting}
In this section we present theoretical insights into AT overfitting and its connection to ERM overfitting. Recall the notations of Section \ref{sec: adv and generalization}. Recall that ERM overfitting gap, between training and testing loss, for a data generation process $\samples_n$ can be written as
\begin{align}
\gap_\ERM(\samples_n, \loss) &\!=\!
\E_{\dist_n \sim \samples_n}%
\!\bigg[\!
        \underbrace{\E_{\dist}[\loss({\theta}_{{\dist}_n},\!Z)]}_{\text{test loss}}
        \!-\!
        \underbrace{
        \E_{{\dist}_n}[\loss({\theta}_{{\dist}_n},\!Z)]}_{\text{train loss}}\!\bigg] \!\!=\!\E_{{\dist}_n^1 \sim \samples_n}\E_{{\dist}_n^2 \sim \samples_n} \!\!
    \left[
        \E_{{\dist}_n^2}[\loss({\theta}_{{\dist}_n^1},\!Z)]
    \!-\!
        \E_{{\dist}_n^2}[\loss({\theta}_{{\dist}_n^2},\!Z)]
    \right]\! \geq \! 0. \label{eq: ERM gap}
\end{align}
Here $\dist_n^2$ can be interpreted as the empirical distribution of random test data and $\dist_n^1$ as the empirical distribution of training data which are equal in distribution but independent. The model $\theta_{\dist_n^1}$ is trained in the context of $\dist_n^1$ and tested with the help of $\dist_n^2$ and subsequently compared to its training performance on $\dist_n^2$. Hence, the \ac{ERM} gap is fundamentally due to the distribution difference (statistical error) between the empirical distribution $\dist_n^1$ of the training data  and the empirical distribution $\dist_n^2$ of the test data. 

\begin{proof}[Proof of Theorem \ref{thm: AT gap}]
  First, observe that we have
  \begin{align*}
    \gap_\adv(\samples_n, \loss) 
    \defeq& 
            \E_{{\dist}_n \sim \samples_n}\left[ 
            \E_{\dist({\theta}^\adv_{{\dist}_n})}[\loss({\theta}^\adv_{{\dist}_n},Z)]
            \right]
            -
            \E_{{\dist}_n \sim \samples_n}\left[ 
            \E_{{\dist}_n({\theta}^\adv_{{\dist}_n})}[\loss({\theta}^\adv_{{\dist}_n},Z)]
            \right] \\
    =&
       \E_{{\dist}_n^1 \sim \samples_n}\E_{{\dist}_n^2 \sim \samples_n}
       \left[
       \E_{{\dist}_n^2({\theta}^\adv_{{\dist}_n^1})}[\loss({\theta}^\adv_{{\dist}_n^1},Z)]
       \right]
       -
       \E_{{\dist}_n^1 \sim \samples_n}\E_{{\dist}_n^2 \sim \samples_n}
       \left[
       \E_{{\dist}_n^2({\theta}^\adv_{{\dist}_n^2})}[\loss({\theta}^\adv_{{\dist}_n^2},Z)]
       \right] \\
    =&
       \E_{{\dist}_n^1 \sim \samples_n}\E_{{\dist}_n^2 \sim \samples_n}
       \left[
       \E_{{\dist}_n^2({\theta}^\adv_{{\dist}_n^1})}[\loss({\theta}^\adv_{{\dist}_n^1},Z)]
       -
       \E_{{\dist}_n^2({\theta}^\adv_{{\dist}_n^2})}[\loss({\theta}^\adv_{{\dist}_n^1},Z)]
       \right]\\
          &\qquad+
            \E_{{\dist}_n^1 \sim \samples_n}\E_{{\dist}_n^2 \sim \samples_n}
            \left[
            \E_{{\dist}_n^2({\theta}^\adv_{{\dist}_n^2})}[\loss({\theta}^\adv_{{\dist}_n^1},Z)]
            -
            \E_{{\dist}_n^2({\theta}^\adv_{{\dist}_n^2})}[\loss({\theta}^\adv_{{\dist}_n^2},Z)]
            \right]
  \end{align*}
  Here the first equality uses Fubini's theorem with $\E_{\dist_n\sim\samples_n}[\dist_n]=\dist$ which applies as $\E_{\dist}[\ell^\cN(\theta, z)]$ is assumed finite for any $\theta\in \Theta$ and we have
  \(
  \E_{\dist(\theta)}[\ell(\theta, z)] = \E_{\dist}[\ell^\cN(\theta, z)] = \E_{\dist_n\sim\samples_n}[\E_{\dist_n}[\ell^\cN(\theta, z)]] = \E_{\dist_n\sim\samples_n}[\E_{\dist_n(\theta)}[\ell(\theta, z)]]
  \)
  for every $\theta\in\Theta$ as well. The second equality follows by adding and subtracting $\E_{{\dist}_n^1 \sim \samples_n}\E_{{\dist}_n^2 \sim \samples_n} [\E_{{\dist}_n^2({\theta}^\adv_{{\dist}_n^2})}[\loss({\theta}^\adv_{{\dist}_n^1},Z)]]$.

  Using the notations ${\dist}_n^1({\theta}^\adv_{{\dist}_n^1}) = \dist_n^{1,\adv}\sim\samples^\adv_n$ and ${\dist}_n^2({\theta}^\adv_{{\dist}_n^2}) = \dist_n^{2,\adv}\sim\samples^\adv_n$ we get
  \begin{align*}
    \gap_\adv(\samples_n, \loss) 
    =&  \E_{{\dist}_n^1 \sim \samples_n}\E_{{\dist}_n^2 \sim \samples_n}
       \left[
       \E_{{\dist}_n^2({\theta}^\adv_{{\dist}_n^1})}[\loss({\theta}^\adv_{{\dist}_n^1},Z)]
       -
       \E_{\dist_n^{2,\adv}}[\loss({\theta}^\adv_{{\dist}_n^1},Z)]
       \right]\\
     &\qquad+
       \E_{\dist_n^{1,\adv}\sim\samples^\adv_n}\E_{\dist_n^{2,\adv}\sim\samples^\adv_n}
       \left[
       \E_{\dist_n^{2,\adv}}[\loss({\theta}_{{\dist}_n^{1,\adv}},Z)]
       -
       \E_{\dist_n^{2,\adv}}[\loss({\theta}_{{\dist}_n^{2,\adv}},Z)]
       \right]\\
    =&  \E_{{\dist}_n^1 \sim \samples_n}\E_{{\dist}_n^2 \sim \samples_n}
       \left[
       \E_{{\dist}_n^2({\theta}^\adv_{{\dist}_n^1})}[\loss({\theta}^\adv_{{\dist}_n^1},Z)]
       -
       \E_{\dist_n^{2,\adv}}[\loss({\theta}^\adv_{{\dist}_n^1},Z)]
       \right] + \gap_\ERM(\samples^\adv_n, \loss)
  \end{align*}
  where we exploit that $\theta^\adv_{\dist'} = \theta_{\dist'^{\adv}}$ for all $\dist'$ due to saddle point condition in Assumption \ref{assumption:saddle-point}.
\end{proof}


\section{Further insights on the holistic robust loss}\label{sec:HR-appendix}
In this section, we detail further insights on the \ac{HR} loss.
Recall that the \ac{HR} loss is defined for all $\theta \in \Theta$ as
\(
  \cost_\HR ^{\cN,\alpha,r}(\theta) \!\defeq\!\max \{ \E_{\dist_\test'}[\loss(\theta,Z)] : \dist_\test'\! \in \mathcal{U}_{\cN,\alpha,r}(\dist_n) \}
\)
with ambiguity set
\[
\mathcal{U}_{\cN,\alpha,r}(\dist_n) \defeq \{ \dist_\test' : \exists \dist' ~\st~ \LP_{\cN}(\dist_n, \dist') \leq \alpha, \, \KL(\dist',\dist_\test') \leq r\}
\]
with $\KL$ denoting here the \acl{KL} divergence $\KL(\dist',\dist_\test') = \int \log(\tfrac{\d\dist'}{\d \dist_\test'})(z) d\dist'(z)$ and \ac{LP} the optimal transport metric $\LP_\cN(\dist_n, \dist') \defeq \inf \{ \int \mathbf{1}(z'-z \not \in \cN) \, \d\gamma(z,z') \; : \; \gamma \in \Gamma(\dist_n,\dist')\}$ where $\Gamma(\dist_n,\dist')$ denotes set of all couplings on $\mathcal Z\times\mathcal Z$ between $\dist_n$ and $\dist'$ \citep{villani2009optimal}.
The \ac{HR} loss on an arbitrary collection of data points $\{z_i=(x_i,y_i)\}_{i=1}^n$ admits (\citet{bennouna2022holistic}, Theorem 3.5) the tractable convex reformulation 
\begin{align}
  \label{eq: HR-finite-formulation-app}
  \cost_\HR^{\cN,\alpha,r}(\theta) =  &
                                        \left\{\begin{array}{r@{~}l}
      \max & \sum_{i =1}^n d'_{i}\loss^{\cN}(\theta,z_i) + d'_{0} \loss^{\mathcal Z}(\theta) \\[0.5em]
      \st &  d'\in \Re^{n+1}_+, \, \hat d'\in \Re^{n+1}_+,\, s^n \in \Re^n_+,\\[0.5em]
           &  \sum_{i=0}^n d'_i = \sum_{i=0}^n \hat{d}'_i = 1,\, \hat{d}'_i \!+\! s_i\!=\!\frac{1}{n}~~\forall i \!\in\! [1,\dots,n],\\[0.5em]
           & \sum_{i=0}^n \hat{d}'_i \log ( \tfrac{\hat{d}_i}{d'_i} ) \leq r,~\sum_{i=1}^n s_i\leq \alpha
    \end{array}\right.
\end{align}
where we recall $\loss^\cN(\theta,z) \defeq \sup_{\delta \in \cN, z+\delta\in \mathcal Z} \loss(\theta,z+\delta)$ for all $z \in \mathcal{Z}$ and $\loss^{\mathcal Z}(\theta)\defeq \sup_{z\in \mathcal Z} \loss(\theta,z)$ for all $\theta \in \Theta$. Recall that the standard \ac{AT} loss is
\(
  \cost_\adv(\theta) = \sum_{i=1}^n d'_{i}\loss^{\cN}(\theta,z_i)
\)
The \ac{HR} loss is hence in essence a smart adversarial weighing of the classical \ac{AT} loss terms but also includes an additional loss term $\loss^{\mathcal Z}(\theta)$. 

In what remains of this section we discuss several interesting special cases of our \ac{HR} loss functions.

\textbf{Robustness to Evasion $(\cN \neq \{0\}, \alpha = 0, r=0)$.} The set $\cN$ characterizes the desired robustness to evasion attacks. Indeed, notice that in the special case of $\alpha = 0$ and $r=0$, we have $\mathcal{U}_{\cN,0,0}(\dist_n) = \{ \dist_\test' \; : \; \LP_{\cN}(\dist_n, \dist_\test') \leq 0\} = \{ \dist_\test' \; : \; \exists \gamma \in \Gamma(\dist_n,\dist_\test')  ~\st~ \int\mathbf{1}(z'-z \not \in \cN) \, \d\gamma(z,z') = 0\}$. Hence,
\begin{equation*}
  \cost_\HR ^{\cN,0,0}(\theta)
  = \int \left[\max_{\delta \in \cN, z+\delta\in \mathcal Z}\loss(\theta, z+\delta)\right] \, \d \dist_n(z)
  =  \E_{\dist_n}[\ell^\cN(\theta, Z)]=
  \cost^\adv(\theta),
\end{equation*}
which is precisely the classical adversarial training loss proposed by \cite{madry2018towards}. Recall that here $Z=(X,Y)$ represents both the covariate and labels. In practice, a common choice for the set $\cN$ is $\{\delta \, : \, \| \delta\|_p\leq \epsilon \} \times \{0\}$ which allows for covariate noise bounded in an $\ell_p$ ball.

\textbf{Robustness to Poisoning $(\cN = \{0\}, \alpha > 0, r=0)$.} The parameter $\alpha$ sets the desired robustness to poisoning attacks when less than $\alpha n$ samples are tampered with by the adversary. Notice that with $\cN = \{0\}$ and $r=0$, we have $\mathcal{U}_{\{0\},\alpha,0}(\dist_n) = \{ \dist_\test' \; : \; \LP_{\{0\}}(\dist_n, \dist_\test') \leq \alpha\} = \{ \dist_\test' \; : \; \TV(\dist_n, \dist_\test') \leq \alpha\}$ where $\TV$ is the total variation distance associated with the optimal transport cost $c(z,z') = \mathbf{1}(z\neq z')$. Hence, 
\[
  \cost_\HR^{\{0\},\alpha,0}(\theta)
  = 
  \max \{ \E_{\dist_\test'}[\loss(\theta,Z)] : \TV(\dist_n,\dist_\test') \leq \alpha \}.
\]

\textbf{Robustness to Statistical Error $(\cN = \{0\}, \alpha = 0, r > 0)$.}
The parameter $r$ sets the desired robustness to statistical error, guaranteeing that the computed loss is an upper bound on the out-of-sample loss with probability $1-e^{-rn + O(1)}$. When $\cN = \{0\}$ and $\alpha = 0$, the \ac{HR} loss becomes
\begin{equation*}
  \cost_\HR^{\{0\},0,r}(\theta)
  = 
  \max \{ \E_{\dist_\test'}[\loss(\theta,Z)] :\; \KL(\dist_n,\dist_\test') \leq r \},
\end{equation*}
which is classical robust formulation extensively considered by \citet{lam2019recovering,namkoong2017variance,vanparys2021data,bennouna2021learning}.

\textbf{Robustness to Evasion and Statistical Error $(\cN \neq \{0\}, \alpha = 0, r>0)$.} The \ac{HR} loss with $\alpha=0$ protects against evasion attacks while ensuring strong generalization. The \ac{HR} loss hence combats the robust overfitting phenomenon of \ac{AT} with KL distributional robustness. In this case, the \ac{HR} loss reduces to 
\[
  \cost_\HR^{\cN,0,r}(\theta)
  = 
  \max \left\{ \E_{\dist_\test'}\left[\loss^\cN(\theta,Z)\right] \; : \; 
  \KL(\dist_n,\dist_\test') \leq r \right\}.
\]

\textbf{Robustness to Poisoning and Statistical Error $(\cN = \{0\}, \alpha > 0, r >0)$.} The \ac{HR} loss with $\cN=\{0\}$ protects against poisoning attacks while ensuring strong generalization. In this case, the \ac{HR} loss becomes 
\begin{equation}\label{eq: TVKL}
  \cost_\HR^{\{0\},\alpha,r}(\theta)
  = 
  \max \{ \E_{\dist_\test'}[\loss(\theta,Z)] : \exists \dist' ~\st~ \TV(\dist_n,\dist')\leq \alpha, \; \KL(\dist',\dist'_\test) \leq r   \}.
\end{equation}

\section{Proof of the robustness certificate of the holistic robust loss}
\label{app:holistic-robust-loss}

In this section, we prove formally Theorem \ref{thm: upper-bound} and its finite sample version Theorem \ref{theorem:main:finite}.
The goal is to show that the HR training loss $\cost_{\HR} ^{\cN,\alpha,r}(\theta^\HR_{\dist_n})$ is an upper bound on the test loss $\cost_\test(\theta^\HR_{\dist_n})$ with high probability.

\begin{proof}[Proof of Theorem \ref{thm: upper-bound}]

  Let us denote with $\alpha'=\alpha-\delta_1$ and $\delta_1>0$ the fraction of the training data corrupted by poisoning and $0\in \cN'\subseteq \interior(\cN)$ the compact evasion attack set.
  Define
  \(
    \delta_2 \defeq \min_{n\in \mathcal Z, n'\in \mathcal Z} \{\|n-n'\| \, : \, n\not\in \interior(\cN), \, n'\in \cN'\}.
  \)
  We remark that as the feasible region is compact and the objective function continuous the minimum is achieved at some $n^\star\neq n'^\star$ and hence $\delta_2 = \|\delta-\delta'\| >0$.
  Set now $\delta'=\min(\delta_1,\delta_2)>0$ and note that $\alpha'+\delta'\leq \alpha$ and $\cN'+B(0, \delta')\subseteq \cN$.
  Theorem \ref{theorem:main:finite} hence guarantees that
  \[
    \Prob\!\left(\cost_{\HR} ^{\cN,\alpha,r}(\theta)\geq \cost_\test(\theta),~\forall\theta\in\Theta \right)\geq 1-\left(\tfrac{4}{\delta'}\right)^{m(\mathcal Z, \delta')}\exp(-rn)
  \]
  for all $n\geq 1$ where $m(\mathcal Z, \delta')=\min \{k\geq 0 \, : \, \exists z_1,\dots, z_k\in \mathcal Z ~\st~\cup_{i=1}^k B(z_i,\delta')\supseteq \mathcal Z\}$ denotes the internal covering number of the support set $\mathcal Z$ with balls of radius $\delta'$.
  Hence, in particular we have for all $n\geq 1$ the bound
  \[
    \Prob\!\left(\cost_{\HR} ^{\cN,\alpha,r}({\theta}^{\HR}_{{\dist}_n})\geq \cost_\test({\theta}^{\HR}_{{\dist}_n}) \right)\geq 1-\left(\tfrac{4}{\delta'}\right)^{m(\mathcal Z, \delta')}\exp(-rn)
  \]
  from which the claim follows immediately.
\end{proof}

To prove the main finite sample result in Theorem \ref{theorem:main:finite} the following composition lemma is useful.

\begin{lemma}
  \label{lemma:composition}
  Let $\dist$ and $\dist'$ be two distributions supported on $\mathcal Z$ and suppose $0\in\cN_1$ and $0\in\cN_2$.
  We have
  \begin{equation*}
    \{ \dist'' \; : \; \exists \dist' ~ \st~  \LP_{\cN_1}(\dist,\dist') \leq \alpha_1, \; \LP_{\cN_2}(\dist',\dist'')\leq \alpha_2 \} = \{ \dist'' \; : \; \LP_{\cN_1+\cN_2}(\dist,\dist'') \leq \alpha_1+\alpha_2 \}
  \end{equation*}
  where here $\cN_1+\cN_2$ denotes the Minkowski sum of $\cN_1$ and $\cN_2$.
\end{lemma}
\begin{proof}
  Define $\mathcal E\defeq \{ \dist'' \; : \; \exists\dist' ~ \st~  \LP_{\cN_1}(\dist,\dist') \leq \alpha_1, \; \LP_{\cN_2}(\dist',\dist'')\leq \alpha_2 \}$ and $\mathcal E'\defeq \{ \dist'' \; : \; \LP_{\cN_1+\cN_2}(\dist,\dist'') \leq \alpha_1+\alpha_2$.
  We first prove that $\mathcal E\subseteq \mathcal E'$. Observe that we have indeed
  \begin{align*}
    \mathcal E=& \{ \dist'' \; : \; \exists Z_1 \sim \dist, \, Z_2\sim \dist', \, Z_3\sim  \dist'' ~\st~ \Prob(Z_2-Z_1\in \cN_1)\geq 1-\alpha_1, \, \Prob(Z_3-Z_2\in \cN_2)\geq 1-\alpha_2 \}\\
    \subseteq & \{ \dist'' \; : \; \exists Z_1 \sim \dist, \, Z_2\sim \dist', \, Z_3\sim \dist'' ~\st~ \Prob(Z_2-Z_1\in \cN_1, \, Z_3-Z_2\in \cN_2)\geq 1-\alpha_1-\alpha_2 \}.\\
    \subseteq &  \{ \dist'' \; : \; \exists Z_1 \sim \dist, \, Z_3\sim \dist'' ~\st~ \Prob(Z_3-Z_1\in \cN_1+\cN_2)\geq 1-\alpha_1-\alpha_2 \} = \mathcal E'
  \end{align*}
  where the first inclusion follows from Fr\'echet inequality $\Prob(A\cap B)\geq \Prob(A)+\Prob(B)-1$.

  We now prove that also $\mathcal E'\subseteq \mathcal E$.
  Consider an arbitrary distribution $\dist'' \in \mathcal E'$. Hence, we can find two random variables $Z_1 \sim \dist$ and $Z_3\sim \dist''$ so that $\Prob(Z_3-Z_1\in \cN_1+\cN_2)\geq 1-\alpha_1-\alpha_2$.
  Suppose that $z_3-z_1\in \cN_1+\cN_2$ then by definition of the Minkowski sum we can always find $z_2$ such that $z_2-z_1\in \cN_1$ and $z_3-z_2\in \cN_2$. Hence, we can construct a random variable $Z'_2$ so that $Z_3-Z_1\in \cN_1+\cN_2 \implies Z_2'-Z_1\in \cN_1, \, Z_3-Z'_2\in \cN_2$. Let $s$ be an independent Bernoulli random variable with success parameter $\tfrac{\alpha_1}{(\alpha_1+\alpha_2)}$. Construct now $Z_2'' = Z_1 (1-s)+Z_3 s$ and let finally $Z_2 = Z_2' \mathbf{1}(Z_3-Z_1\in \cN_1+\cN_2) + Z_2''\mathbf{1}(Z_3-Z_1\not\in \cN_1+\cN_2)$. Standard manipulations guarantee that
  \begin{align*}
    \Prob(Z_2-Z_1\in \cN_1) = & \Prob(Z_2' \mathbf{1}(Z_3-Z_1\in \cN_1+\cN_2) + Z_2''\mathbf{1}(Z_3-Z_1\not\in \cN_1+\cN_2)-Z_1\in \cN_1)\\
    = & \Prob(Z'_2-Z_1\in \cN_1, \, Z_3-Z_1\in \cN_1+\cN_2)+\Prob(Z_2'' - Z_1 \in \cN_1, \, Z_3-Z_1\not\in \cN_1+\cN_2)\\
    = & \Prob(Z_3-Z_1\in \cN_1+\cN_2) +\Prob( Z_2''-Z_1\in \cN_1, \, Z_3-Z_1\not\in \cN_1+\cN_2)\\
    = & \Prob(Z_3-Z_1\in \cN_1+\cN_2) +\Prob(s=0, \, Z_3-Z_1\not\in \cN_1+\cN_2)\\
    = & \Prob(Z_3-Z_1\in \cN_1+\cN_2) +\Prob(s=0)\Prob(Z_3-Z_1\not\in \cN_1+\cN_2)\\
    = & \Prob(Z_3-Z_1\in \cN_1+\cN_2) +\Prob(s=0)(1-\Prob(Z_3-Z_1\in \cN_1+\cN_2))\\
    = & \Prob(s=0) + \Prob(Z_3-Z_1\in \cN_1+\cN_2)\Prob(s=1)\\
    \geq & \left(1- \tfrac{\alpha_1}{(\alpha_1+\alpha_2)}\right)+(1-\alpha_2-\alpha_1)\tfrac{\alpha_1}{(\alpha_1+\alpha_2)} =  1-\alpha_1
  \end{align*}
  and similarly
    \begingroup
  \allowdisplaybreaks

    \begin{align*}
      \Prob(Z_3-Z_2\in \cN_2) = & \Prob(Z_3-Z_2' \mathbf{1}(Z_1-Z_3\in \cN_1+\cN_2) - Z_2''\mathbf{1}(Z_1-Z_3\not\in \cN_1+\cN_2)\in \cN_2)\\
      = & \Prob(Z_3-Z_2'\in \cN_2, \, Z_3-Z_1\in \cN_1+\cN_2)+\Prob( Z_3-Z_2''\in \cN_2, \, Z_3-Z_1\not\in \cN_1+\cN_2)\\
      = & \Prob(Z_3-Z_1\in \cN_1+\cN_2) +\Prob(Z_3-Z_2''\in \cN_2, \,  Z_3-Z_1\not\in \cN_1+\cN_2)\\
      = & \Prob(Z_3-Z_1\in \cN_1+\cN_2) +\Prob(s=1, \,  Z_3-Z_1\not\in \cN_1+\cN_2)\\
      = & \Prob(Z_3-Z_1\in \cN_1+\cN_2) +\Prob(s=1)\Prob( Z_3-Z_1\not\in \cN_1+\cN_2)\\
      = & \Prob(Z_3-Z_1\in \cN_1+\cN_2) +\Prob(s=1)(1-\Prob( Z_3-Z_1\in \cN_1+\cN_2))\\
      = & \Prob(s=1) + \Prob( Z_3-Z_1\in \cN_1+\cN_2)\Prob(s=0)\\
      \geq & \left(1-\tfrac{\alpha_2}{(\alpha_2+\alpha_1)}\right) + (1-\alpha_2-\alpha_1)\tfrac{\alpha_2}{(\alpha_2+\alpha_1)} = 1-\alpha_2.
  \end{align*}
  \endgroup
  Hence we have that $\dist''\in \mathcal E$ as well. As $\dist''$ was arbitrary we also have that $\mathcal E'\subseteq \mathcal E$.
\end{proof}

The goal now is to show that a slightly inflated \ac{KL} ball captures the deviation of an empirical distribution from the out-of-sample distribution that generated it.
Consider the random empirical distribution $\dist''_n$ of $n$ (uncorrupted) independent samples from a data generating distribution $\dist$.
It is well known that if a data generating distribution is continuous then the \ac{KL} ball \( \{ \dist ''\, : \, \KL( \dist''_n, \dist'')\leq r \)\} fails to be a confidence set of the data generating distribution $\dist$.
Indeed, the \ac{KL} divergence between any distribution supported on a finite set of points and a continuous distribution is unbounded from above.
Consequently, if the generating distribution $\dist$ is continuous we have \( \Prob(\dist \in \{ \dist'' \, : \,\KL(\dist''_n, \dist'')\leq r \}) = 0\) for any $r\in \Re$.

The following results however states that a slightly inflated version of an appropriately sized \ac{KL} ball around the empirical distribution of a clean training data set does serve as a confidence region of the data generating distribution.

\begin{lemma}[{\citet{dembo2009large}}]
  \label{lemma:finite_sample}
  Let $\dist''_n$ be the empirical distribution of $n$ independent samples with distribution $\dist$ supported on a compact set $\mathcal Z$. Then, for all $\delta>0$
  \[
    \Prob(\dist \in \{ \dist'' \, : \, \exists \dist'~\st~ \LP_{B(0, \delta)}(\dist''_n, \dist')\leq \delta, ~\KL(\dist', \dist'')\leq r \})\geq 1-\left(\frac{4}{\delta}\right)^{m(\mathcal Z, \delta)}\exp(-rn)
  \]
  where $m(\mathcal Z, \delta)=\min \{k\geq 0 \, : \, \exists z_1,\dots, z_k\in \mathcal Z ~\st~\cup_{i=1}^k B(z_i,\delta)\supseteq \mathcal Z\}$ denotes the internal covering number of the support set $\mathcal Z$.
\end{lemma}
\begin{proof}
For any given set $\mathcal A$,
let $\mathcal A^\delta \defeq \{\dist' \, :\, \pi(\dist'', \dist')\leq \delta,\, \dist'' \in \mathcal A\}$ denote the $\epsilon$-inflation of the set $\mathcal A$ where $\pi$ denotes here the \acl{LP} distance \cite{huber1981robust}.
  We remark that here the \ac{LP} balls $\mathcal B(\dist') \defeq \{\dist' \, :\, \pi(\dist'', \dist')\leq \delta,\, \dist'' \in \mathcal A\}$ are compact as $\pi$ is continuous in the weak topology and $\mathcal Z$ is compact \citep{prokhorov1956convergence}.
  \citet[Exercise 4.5.5]{dembo2009large} hence establish using a covering argument that for any set $\mathcal A$ and $\delta>0$ we have for all $n\geq 1$ the upper bound
  \begin{align}\label{eq:inclusion-proba}
    \Prob(\dist''_n \in \mathcal A) \leq & m_\LP(\mathcal A , \delta) \exp\left(- n \textstyle\inf_{\dist' \in \mathcal A^\delta} {\KL}(\dist', \dist)\right)
  \end{align}
  where $m_\LP(\mathcal A, \delta)=\min \{k\geq 0 \, : \, \exists \dist_1,\dots, \dist_k\in \mathcal A ~\st~\cup_{i=1}^k \mathcal B(\dist_i,\delta)\supseteq \mathcal A \}$ denotes the internal covering number of the set $\mathcal A$ with \ac{LP} balls of radius $\delta$.

  \citet[Exercise 6.2.19]{dembo2009large} also establish that with respect to the \acl{LP} distance the covering number of a subset $\mathcal A$ of the probability simplex on $\mathcal Z$ satisfies for all $\delta> 0$ the inequality
\(
  m_{\LP}(\mathcal A, \delta)\leq \left(\tfrac{4}{\delta}\right)^{m(\mathcal Z ,\delta)}
\)
where here $m(\mathcal Z,\delta)$ in turn denotes the covering number of the compact event set $\mathcal Z$ itself.
Consider now a particular set $\mathcal A$ with complement $\{ \dist'' \, : \, \exists \dist'~\st~\pi(\dist'', \dist')\leq \delta, ~\KL(\dist', \dist)\leq r \}$.
We hence have 
\begingroup
\allowdisplaybreaks
\begin{align*}
  & \Prob(\dist \in \{ \dist'' \, : \, \exists \dist'~\st~ \LP_{B(0, \delta)}(\dist''_n, \dist')\leq \delta, ~\KL(\dist', \dist'')\leq r \})\\
  = & \Prob(\dist \in \{ \dist'' \, : \, \exists \dist'~\st~ \pi(\dist''_n, \dist')\leq \delta, ~\KL(\dist', \dist'')\leq r \})\\
  = & \Prob( \dist''_n \in \{ \dist'' \, : \, \exists \dist'~\st~\pi(\dist'', \dist')\leq \delta, ~\KL(\dist', \dist)\leq r \})\\
  = & 1- \Prob( \dist''_n \not\in \{ \dist'' \, : \, \exists \dist'~\st~\pi(\dist'', \dist')\leq \delta, ~\KL(\dist', \dist)\leq r \})\\
  = & 1- \Prob( \dist''_n \in \mathcal{A})\\
  \geq & 1 - m_\LP(\mathcal A , \delta) \exp\left(- n \textstyle\inf_{\dist' \in \mathcal A^\delta} {\KL}(\dist', \dist)\right)\\
  \geq & 1-\left(\tfrac{4}{\delta}\right)^{m(\mathcal Z,\delta)} \exp(-r n).
\end{align*}
\endgroup
The first equality is due to a classical result of \citet{strassen1965existence} that implies $\{\dist': \; \LP_{B(0, \delta)}(\dist''_n, \dist') \leq \delta\} = \{\dist': \; \pi(\dist''_n, \dist') \leq \delta\}$. The first inequality uses the upper bound \eqref{eq:inclusion-proba}.
The ultimate inequality follows from the implication $\dist' \in \mathcal A^\delta\implies \KL(\dist', \dist)>r$ which we now prove by contradiction.
Suppose indeed that there exists $\dist''\in \mathcal A^\delta$ and $\KL(\dist'', \dist)\leq r$. From the definition of $\mathcal A^\delta$ there exists $\dist'\in \mathcal A$ such that
$\pi(\dist'', \dist')=\pi(\dist', \dist'')\leq \delta$.
As we have both $\pi(\dist', \dist'')\leq \delta$ and $\KL(\dist'', \dist)\leq r$ it follows that $\dist'$ is also  in the complement of $\mathcal A$; we have reached a contradiction.
\end{proof}

The support function of a set $\mathcal A$ is defined as the function \( h_{\mathcal A}(\ell) \defeq \sup_{\dist'\in \mathcal A} \int \ell(z) \, \d \dist'(z) \).
For all distribution $\dist'$, denote $\mathcal E(\dist') \defeq \{ \dist_\test' \, : \,  \exists \dist \, \st \, \KL(\dist', \dist)\leq r, \, \LP_{\cN}(\dist, \dist_\test')\leq 0 \}$ and $\mathcal E'(\dist')\defeq  \{ \dist_\test' \, :\ \exists \dist \, \st  \, \LP_{\cN}(\dist', \dist)\leq 0, \, \KL(\dist, \dist_\test')\leq r \}$.
We show that both sets have the same support function.

\begin{lemma}
  \label{lemma:support_functions}
  For all distribution $\dist'$, the support function of the sets $\mathcal E(\dist')$ and $\mathcal E'(\dist')$ coincide, i.e.,
  \(
    h_{\mathcal E(\dist')}(\ell) = h_{\mathcal E'(\dist')}(\ell)
  \)
  for all $\ell$ with $-\infty<\inf_{z\in \mathcal Z} \ell(z) \leq \sup_{z\in \mathcal Z} \ell(z) <\infty$ and $0\in \cN$.
\end{lemma}
\begin{proof}
  Remark that if $r=0$ the result is trivial as then clearly $\mathcal E(\dist')=\mathcal E'(\dist')=\{ \dist_\test' \, : \, \LP_{\cN}(\dist', \dist_\test')\leq 0 \}$.
  Let now $r>0$ and assume without loss of generality that $\inf_{z\in Z}\ell(z)\geq 0$. Indeed, we may remark that
  $h_{\mathcal E(\dist')}(\ell-a) = h_{\mathcal E'(\dist')}(\ell-a)=h_{\mathcal E(\dist')}(\ell) - a= h_{\mathcal E'(\dist')}(\ell)-a$ for all $a\in\Re$ and we can consider for instance $a = \inf_{z\in \mathcal Z}\ell(z)$.
  
  Remark first that for all $\loss$, we have
  \begin{align*}
    & h_{\mathcal E(\dist')}(\ell)\\
    := & \sup\{ \textstyle\int \ell(z) \, \d \dist_\test'(z)\, : \, \KL(\dist', \dist)\leq r, \, \LP_{\cN}(\dist, \dist_\test')\leq 0 \}\\
    =    & \sup \{\textstyle \int \ell^{\cN}(z) \, \d \dist(z) \, : \, \KL(\dist', \dist)\leq r \}\\
    = & \inf \{\textstyle \alpha-\exp(\int \log(\alpha-\ell^{\cN}(z))\, \d \dist'(z)-r) : \sup_{z\in\mathcal Z}\ell^{\cN}(z) \leq \alpha \leq  \sup_{z\in\mathcal Z}\ell^{\cN}(z)/(1-e^{-r})\} \\
    = & \inf \{\textstyle \alpha-\exp(\int \log(\alpha-\ell^{\cN}(z))\, \d \dist'(z)-r) : \sup_{z\in\mathcal Z}\ell(z) \leq \alpha \leq  \sup_{z\in\mathcal Z}\ell(z)/(1-e^{-r})\}.
  \end{align*}
  The first equality is due to the observation that as remarked before $\max\{\int \loss(z)\, \d \dist'_\test(z) : \dist'_\test \in \mathcal{U}_{\cN}(\dist')\} = \int \loss^\cN(z) \, \d\dist'(u)$ for any distribution $\dist'$. The second equality follows from a dual representation of the \ac{KL} ball given in Proposition 5 by \citet{vanparys2021data} and holds for $r>0$. The final equality follows from $\sup_{z\in\mathcal Z}\ell^{\cN}(z) = \sup_{z\in\mathcal Z}\ell(z)$ as $0\in \cN$.
  Second, we observe that
  \begin{align*}
    & h_{\mathcal E'(\dist')}(\ell)\\
    := & \sup\{ \textstyle\int \ell(z) \, \d \dist_\test'(z)\, : \, \LP_{\cN}(\dist', \dist)\leq 0, \, \KL(\dist, \dist_\test')\leq r \}\\
    =    & \sup \{ \inf \{\textstyle \alpha-\exp(\int \log(\alpha-\ell(z))\, \d \dist(z)-r) \, : \, \sup_{z\in\mathcal Z}\ell(z) \leq \alpha \leq  \sup_{z\in\mathcal Z}\ell(z)/(1-e^{-r})\}  \, : \, \LP_{\cN}(\dist', \dist) \leq 0 \}\\
    =    & \inf \{ \sup \{\textstyle\alpha-\exp(\int \log(\alpha-\ell(z))\, \d \dist(z)-r) \, : \, \LP_{\cN}(\dist', \dist) \leq 0\}  \, :\, \sup_{z\in\mathcal Z}\ell(z) \leq \alpha \leq  \sup_{z\in\mathcal Z}\ell(z)/(1-e^{-r}) \}\\
    =    & \inf \{ \textstyle\alpha-\exp(\inf\{ \int \log(\alpha-\ell(z))\, \d \dist(z)\, : \, \LP_{\cN}(\dist', \dist) \leq 0\}-r)   \, :\, \sup_{z\in\mathcal Z}\ell(z) \leq \alpha \leq  \sup_{z\in\mathcal Z}\ell(z)/(1-e^{-r}) \}\\
    = & \min \{\textstyle\int \lambda \log(\tfrac{\lambda}{(\ell^{\cN}(z)-\eta)})\, \d \dist'(z) + (r-1)\lambda +\eta \, : \,\eta\in\Re, ~\lambda\in\Re_+, ~\textstyle\sup_{z\in\mathcal Z}\ell(z)\leq \eta\}.
  \end{align*}
  Here, the first equality from the same dual representation of the \ac{KL} ball as used before.
  We remark that the objective function is convex and lower semicontinuous in $\alpha$ and concave in $\dist'$. Hence, the second equality follows hence from a standard minimax theorem (c.f., Theorem 4.2 in \citet{sion1958general}).
  The third equality uses the fact that the exponential function is nondecreasing.
  The final equality follows similarly from the fact the logarithm is an non-increasing function. Indeed, we can write
  $\inf\{ \textstyle\int \log(\alpha-\ell(z))\, \d \dist(z)\, : \, \LP_{\cN}(\dist', \dist) \leq 0\}
  =  - \sup \{ \textstyle\int -\log(\alpha-\ell(z))\, \d \dist(z)\, : \, \LP_{\cN}(\dist', \dist) \leq 0\}$
  which after using again $\max\{\int \loss(z)\, \d \dist'_\test(z) : \dist'_\test \in \mathcal{U}_{\cN}(\dist')\} = \int \loss^\cN(z) \, \d\dist'(u)$ yields $- \textstyle \int \sup_{\delta\in \cN, z+\delta\in \mathcal Z} -\log(\alpha-\ell(z+\delta))\, \d \dist'(z) = - \textstyle \int -\log(\alpha-\sup_{\delta\in \cN, z+\delta\in \mathcal Z}\ell(z+\delta))\, \d \dist'(z)=   \textstyle\int \log(\alpha-\ell^\cN(z))\, \d \dist(z)\, : \, \LP_{\cN}(\dist', \dist) \leq 0\}$.
\end{proof}

We are finally ready to state and prove the main result.

\begin{theorem}
  \label{theorem:main:finite}
  Consider a margin $\delta>0$. Suppose at most a fraction $\alpha'$ with $\alpha'+\delta \leq \alpha$ of the \ac{IID} training data is poisoned and the evasion attack on the test set is limited to a set $0\in \cN'$ with $\cN'+B(0, \delta) \subseteq \cN$.
  Denote the \acl{HR} solution as $\theta^{\HR}_{\dist_n} \in \argmin_{\theta\in \Theta} \max_{\dist_\test'\in \mathcal{U}_{\cN,\alpha,r}(\dist_n)}\E_{\dist_\test'}[\loss(\theta,Z)].$
  Then, we have
\[
  \Prob\!\left(\cost_{\HR} ^{\cN,\alpha,r}(\theta)\geq \cost_\test(\theta),~\forall \theta\in\Theta \right)\geq 1-\left(\tfrac{4}{\delta}\right)^{m(\mathcal{Z}, \delta)}\exp(-rn)
\]
for all $n\geq 1$, where $m(\mathcal{Z}, \delta)=\min \{k\geq 0 \, : \, \exists z_1,\dots, z_k\in \mathcal Z ~\st~\cup_{i=1}^k B(z_i,\delta)\supseteq \mathcal Z\}$ denotes the internal covering number of the support set $\mathcal Z$.
\end{theorem}
\begin{proof}
  Denote again with $\dist''_n$ the empirical distribution of the training data before poisoning corruption and likewise denote with $\dist$ the data generating distribution of our \ac{IID} samples and $\dist_\test$ the corrupted test, following Theorem 2.1 in \citet{bennouna2022holistic} we have 
  \begin{equation}
    \label{eq:inclusion-i}
    \dist''_n\in \{ \dist'' \, : \, \LP_{\{0\}}(\dist_n, \dist'') \leq \alpha'\}, ~\dist_\test\in  \{ \dist'_\test \, : \, \LP_{\cN'}(\dist, \dist'_\test) \leq 0\}.
  \end{equation}
  Furthermore, according to Lemma \ref{lemma:finite_sample} we have that
  \begin{equation}
    \label{eq:inclusion-iii}
    \Prob(\dist \in \{ \dist''' \, : \, \exists \dist'~\st~ \LP_{B(0, \delta)}(\dist''_n, \dist')\leq \delta, ~\KL(\dist', \dist''')\leq r \})\geq 1-\left(\tfrac{4}{\delta}\right)^{m(\mathcal Z, \delta)}\exp(-rn).
  \end{equation}

  Chaining together the previous inclusions will give us the desired result.
  Let us indeed denote here $\mathcal{U}'_{\cN',\alpha',r}(\dist_n)
  \defeq  \{\dist_\test' \, : \, \exists \dist', \dist'', \dist''' \, \st \, \LP_{\{0\}}(\dist_n, \dist'') \leq \alpha', \, \LP_{B(0, \delta)}(\dist'', \dist')\leq \delta,\, \KL(\dist', \dist''')\leq r, \, \LP_{\cN'}(\dist''', \dist_\test')\leq 0\}$
  Chaining together inclusions \eqref{eq:inclusion-i} and \eqref{eq:inclusion-iii} we get that
  \(
  \Prob(\dist_\test \in {\mathcal{U}}'_{\cN',\alpha',r}(\dist_n))\geq 1-\exp(-rn+m(\mathcal Z, \delta)\log\left(\tfrac{4}{\delta}\right)).
  \)
  From Lemma \ref{lemma:composition} we have furthermore that the previous set can be characterized as
  $\mathcal{U}'_{\cN',\alpha',r}(\dist_n)= \{\dist_\test' \, : \, \exists \dist', \dist'''\, \st \, \LP_{B(0, \delta)}(\dist_n, \dist')\leq \alpha'+\delta,\, \KL(\dist', \dist''')\leq r, \, \LP_{\cN'}(\dist''', \dist_\test')\leq 0\}$

  Consider now the associated robust loss
  \(
    \cost_{\HR'}^{\cN',\alpha',r}(\theta) \!\defeq\!\max \{ \E_{\dist_\test'}[\loss(\theta,Z)]\!:\!\dist_\test'\! \in \mathcal{U}'_{\cN',\alpha',r}(\dist_n) \}
  \).
  Recall that $\cost_\test(\theta) = \E_{\dist_\test}[\loss(\theta,Z)]$ and hence it follows immediately that
  \[
    \Prob(\cost_{\HR'} ^{\cN',\alpha',r}(\theta)\geq \cost_\test(\theta), \; \forall \theta \in \Theta)\geq 1-\exp(-rn+m(\mathcal Z, \delta)\log\left(\tfrac{4}{\delta}\right)).
  \]
  It remains to show that $\cost_{\HR'} ^{\cN',\alpha',r}(\theta) \leq \cost_{\HR} ^{\cN,\alpha,r}(\theta)$ for all $\theta\in \Theta$. 
  We have
  \begin{align*}
    \cost_{\HR'}^{\cN',\alpha',r}(\theta)\defeq & \textstyle\sup_{\LP_{B(0, \delta)}(\dist_n, \dist')\leq \alpha'+\delta} \sup \{ \E_{\dist_\test'}[\loss(\theta,Z)] :  \exists \dist''', \dist_\test'~\st~ \KL(\dist', \dist''')\leq r, \, \LP_{\cN'}(\dist''', \dist_\test')\leq 0\} \\
    =& \textstyle\sup_{\LP_{B(0, \delta)}(\dist_n, \dist')\leq \alpha'+\delta}  \sup \{ \E_{\dist_\test'}[\loss(\theta,Z)] :  \exists \dist''', \dist_\test'~\st~ \LP_{\cN'}(\dist', \dist''')\leq 0, \, \KL(\dist''', \dist_\test')\leq r\} \\
    =& \textstyle \sup \{ \E_{\dist_\test'}[\loss(\theta,Z)] :  \exists \dist''', \dist_\test'~\st~ \LP_{\cN'+B(0, \delta)}(\dist_n, \dist''')\leq \alpha'+\delta, \, \KL(\dist''', \dist_\test')\leq r\} \leq  \cost_{\HR}^{\cN,\alpha,r}(\theta).
  \end{align*}
  The first equality follows from Lemma \ref{lemma:support_functions}. The second equality follows again from Lemma \ref{lemma:composition}. The final inequality follows from the assumption that $\alpha'+\delta\leq \alpha$ and $\cN'+B(0, \delta)\subseteq \cN$.
\end{proof}

Denote with $\cost_{\HR}^{\cN,\alpha,r}(\theta, \dist_n)$ the \ac{HR} cost in which we make its dependence on the distribution of the training data explicit.
Theorem \ref{theorem:main:finite} indicates that this cost is with high probability $1-\exp(-nr+O(1))$ an upper bound on the adversarial test loss $\cost_\test(\theta)$ uniformly over $\theta\in\Theta$.
The next result indicates that our HR cost is (in some sense) tight under relatively mild assumption.

\begin{theorem}
  \label{theorem:main:tight}
  Let $0\in\interior(\cN)$, $\closure(\interior(\cN))=\cN$, $\alpha>0$, $r>0$ and $z\mapsto \ell(\theta, z)$ continuous and bounded for all $\theta\in \Theta$. Consider a function $\cost_A^{\cN,\alpha,r}$. Suppose that for some $\theta_0\in \Theta$ and $\dist'_0\in \cup_{n\geq 1}\{ \sum_{i=1}^n\delta_{z_i}/n \ \, : \, z_i\in \mathcal Z ~~\forall i\in [1,\dots, n] \}$ we have $$\cost_{\HR}^{\cN,\alpha,r}(\theta_0, \dist'_0)> \cost_A^{\cN,\alpha,r}(\theta_0, \dist'_0).$$ Then, there exists a data generating distribution $\dist$ and an adversary which poisons less than a fraction $\alpha$ of all training data points and perturbs the testing data with evasion bounded in a compact set in the interior of $\cN$ for which we have
  \[
    \limsup_{n\to\infty} \frac{1}{n}\log \Prob\!\left(\cost_A ^{\cN,\alpha,r}(\theta) < \cost_\test(\theta), ~\exists\theta\in \Theta \right) \geq \limsup_{n\to\infty} \frac{1}{n}\log  \Prob\!\left(\cost_A ^{\cN,\alpha,r}(\theta_0) < \cost_\test(\theta_0)\right) > -r.
  \]
\end{theorem}

\begin{proof}
  Consider here $n_0\geq 1$ so that we have $\dist'_0\in \{ \sum_{i=1}^{n_0}\delta_{z_i}/n_0 \ \, : \, z_i\in \mathcal Z(n_0)\}$ for some $\mathcal Z(n_0)\defeq \{ z_i\in \mathcal Z\}_{i=1}^{n_0}$.
  In other words, the distribution $\dist'_0$ is the empirical distribution of $n_0$ data points in $\mathcal Z$. Let here $\epsilon\defeq \cost_{\HR}^{\cN,\alpha,r}(\theta_0, \dist'_0)-\cost_A^{\cN,\alpha,r}(\theta_0, \dist'_0)>0$. Without loss of generality we assume that we have ordered the points in $\mathcal Z(n_0)$ so that $\ell^\cN(\theta_0, z_1)\leq \dots \leq \ell^\cN(\theta_0, z_{n_0})$.

  We have by definition
  \[
    \cost_\HR ^{\cN,\alpha,r}(\theta_0; \dist'_0) \!\defeq\!\sup \{ \E_{\dist_\test'}[\loss(\theta_0,Z)] :  \exists \dist', \dist_\test' ~\st~ \LP_{\cN}(\dist'_0, \dist') \leq \alpha, \, \KL(\dist',\dist'_\test) \leq r\}.
  \]
  We first remark that as both the $\LP$ and $\KL$ metrics are jointly convex in the optimization variables $\dist'$ and $\dist_\test'$ and the objective $\E_{\dist_\test'}[\loss(\theta_0,Z)]$ a linear function of the variable $\dist'_\test$ the \ac{HR} loss $\cost_\HR ^{\cN,\alpha,r}(\theta_0; \dist'_0)$ is a concave function of the robustness parameters $\alpha$ and $r$.
  As concave functions are continuous on the interior of their domain the \ac{HR} loss is a continuous function at $r>0$ and $\alpha>0$.
  Hence, we can consider $0<r'<r$ and $\alpha\in \mathbb Q$, $0<\alpha'<\alpha$ so that
  \(
  \cost_\HR ^{\cN,\alpha',r'}(\theta_0; \dist'_0) + \epsilon/4 \geq \cost_\HR ^{\cN,\alpha,r}(\theta_0; \dist'_0).
  \)
  Consider an integer $k'\geq 1$ such that both $n'_0\defeq n_0k'\geq 1$ and $n'_0\alpha'= n_0k'\alpha'$ are integers.
  This choice of $k'$ implies that $\dist'_0$ can be seen as the empirical distribution of $n'_0$ data points in $\mathcal Z$ and our adversary will be able to precisely corrupt $n'_0\alpha'$ data points.
  
  Let us consider now a corrupted version of the support set $\mathcal Z'(n_0)$ defined here as $$\textstyle \mathcal Z'(n_0) = \{ z'_{i}\defeq z_{i} + \argmax_{\delta\in \cN, z_{i}+\delta\in \mathcal Z} \ell(\theta_0, z_{i}+\delta) \}_{i=1}^{n_0} \cup \{z'_\infty\in \argmax_{z\in \mathcal Z}\ell(\theta_0, z)\}$$
  which is well defined as $z\mapsto \ell(\theta_0, z)$ is continuous and $\cN$ compact. Each of the points $z_i'\in \mathcal Z'(n_0)$ represents an evasion corrupted version of its counterpart $z_i\in \mathcal Z(n_0)$ while the point $z'_\infty$ can be interpreted as a poisoned data point maximizing the loss.
  As the empirical distribution $\dist'_0$ is finitely supported, from Lemma 3.7 and Corollary 2.3 in \citet{bennouna2022holistic} it follows that for
  \begin{equation}
    \label{eq:opt-form}
    \textstyle\dist'^\star=\sum_{i=\tau+1}^{n_0} \delta_{z'_i} \dist'_0(z_i)+(1-\alpha-\sum_{i=\tau+1}^{n_0} \dist'_0(z_i))\delta_{z'_\tau}+\alpha \delta_{z'_\infty},
  \end{equation}
  with $\tau$ the smallest integer so that $1-\alpha-\sum_{i=\tau+1}^{n_0} \dist'_0(z_i)>0$, we can find a distribution $\dist'^\star_\test \in \{ \dist' \, : \, \supp(\dist')\subseteq \mathcal Z'(n_0) \}$
  with $$\LP_{\cN}( \dist'_0, \dist'^\star) \leq \alpha', \, \KL(\dist'^\star,\dist'^\star_\test)\leq r', ~\cost_\HR ^{\cN,\alpha',r'}(\theta_0; \dist'_0)=\E_{\dist_\test'^\star}[\loss(\theta_0,Z)].$$ In other words, the distributions $\dist'^\star$ and $\dist'^\star_\test$ represent maximizers in the optimization problem characterizing the HR cost $\cost_\HR ^{\cN,\alpha',r'}(\theta_0; \dist'_0)$. As $n'_0\alpha'$ is integer and  $\dist'_0\in \{ \sum_{i=1}^{n_0}\delta_{z_i}/n_0 \ \, : \, z_i\in \mathcal Z(n_0)\}$ it also follows immediately that
  $\dist'^\star\in  \{ \dist' \, : \, \supp(\dist')\subseteq \mathcal Z'(n_0), \, \dist'(z) n'_0 \in [0, \dots, n'_0] ~~\forall z\in \mathcal Z'(n_0) \}$.
  This last observation implies that also the distribution $\dist'^\star$, just as $\dist'_0$, can be seen as the empirical distribution of $n_0'$ points in $\mathcal Z$.

  Introduce a distribution $\bar \dist'^\star$ supported on $\mathcal Z(n_0)\cup \{z'_\infty\}$ by defining $\bar \dist'^\star(z_i) = \dist'^\star(z_i')$ for all $i\in [1,\dots, n_0]$ and $\bar \dist '^\star(z'_\infty) = \dist'^\star(z_\infty')$. In essence, we move each point in the support of $\dist'^\star$ from a perturbed position $z_i'$ back to $z_i$ with the exception of $z'_\infty$ which we leave as is. From Equation \eqref{eq:opt-form} we have
  \(
  \alpha'\geq \LP_{\cN}(\dist'_0, \dist'^\star) = \dist'^\star(z_\infty') = \bar \dist'^\star(z_\infty') = \LP_{\{0\}}(\dist'_0, \bar \dist'^\star).
  \)
  Likewise, define a distribution $\bar \dist''^\star$ supported on $\mathcal Z(n_0)\cup \{z'_\infty\}$ though $\bar \dist ''^\star(z_i) = \dist'^\star_\test(z_i')$ for all $i\in [1,\dots, n_0]$ and $\bar \dist''^\star(z'_\infty) = \dist'^\star_\test(z_\infty')$. In essence, we move each point in the support of $\dist'^\star_\test$ from a perturbed position $z_i'$ back to $z_i$ with the exception of $z'_\infty$ which we leave as is. It is easy to verify that by construction
  \(
  \KL(\bar \dist'^\star,\bar \dist''^\star)
  =
  \sum_{i=1}^{n_0} \log(\tfrac{\bar \dist'^\star(z_i)}{\bar \dist''^\star(z_i)}) \bar \dist'^\star(z_i)+\log(\tfrac{\bar \dist'^\star(z'_\infty)}{\bar \dist''^\star(z'_\infty)}) \bar \dist'^\star(z'_\infty)
  =
  \sum_{i=1}^{n_0} \log(\tfrac{\dist'^\star(z'_i)}{\dist'^\star_\test(z'_i)}) \dist'^\star(z'_i)+\log(\tfrac{\dist'^\star(z'_\infty)}{\dist'^\star_\test(z'_\infty)})  \dist'^\star(z'_\infty)
  =
  \KL(\dist'^\star,\dist'^\star_\test) \leq r'
  \)
  and
  \(
  \LP_{\cN}(\bar \dist''^\star, \dist'^\star_\test) = 0.
  \)
  In fact, as $z\mapsto \ell(\theta_0, z)$ is continuous and $\closure(\interior(\cN))=\cN$ we can find points $z_i''$ with $z_i''-z_i\in \interior(\cN)$ so that for $\dist''^\star_\test$ defined as $\dist''^\star_\test(z_i'')=\dist'^\star_\test(z_i')$, $i\in [1,\dots, n_0]$ satisfies $\LP_{\cN'}(\bar \dist''^\star, \dist'^\star_\test) = 0$ and $$\cost_\HR ^{\cN,\alpha',r'}(\theta_0; \dist'_0)
  =
  \E_{\dist_\test'^\star}[\loss(\theta_0,Z)]
  \leq \E_{\dist_\test''^\star}[\loss(\theta_0,Z)]+\epsilon/4$$ with $\cN' \defeq \{0\}\cup\{ \delta'_i \defeq z_i''-z_i\}_{i=1}^{n_0}\subseteq \interior(\cN)$. In essence, the points $z_i''$ are chosen to present a perturbation of the points $z_i$ of an evasion adversary limited to a evasion set within a compact subset in the interior of $\cN$.

Consider the data generation process with distribution $\bar \dist''^\star$, where $n\geq 1$ uncorrupted data points, of empirical distribution $\dist''_n$, are sampled from $\bar \dist''^\star$ supported on $\mathcal Z(n_0)\cup \{z'_\infty\}$. First, for this process, if an evasion adversary implements test data evasion with $\cN'$ then by construction of $\dist_\test''^\star$ we get $\cost_\test(\theta_0)\geq  \E_{\dist_\test''^\star}[\loss(\theta_0,Z)]$ simply by adding perturbation $\delta_i'$ when encountering realization $z_i$, $i\in[1,\dots, n_0]$ and $0$ in case $z'_\infty$ is observed.
  Remark that from $\dist'^\star\in  \{ \dist' \, : \, \supp(\dist')\subseteq \mathcal Z'(n_0), \, \dist'(z) n'_0 \in [0, \dots, n'_0] ~~\forall z\in \mathcal Z'(n_0) \}$ and the construction of $\bar \dist'^\star$ it follows that $\bar \dist'^\star\in  \{ \dist' \, : \, \supp(\dist')\subseteq \mathcal Z(n_0) \cup \{z'_\infty\}, \, \dist'(z) n'_0 \in [0, \dots, n'_0] ~~\forall z\in \mathcal Z(n_0) \cup \{z'_\infty\}\}$. In other words, the distribution $\bar \dist'^\star$ can be seen as the empirical distribution of $n_0'$ data points in $\mathcal Z(n_0) \cup \{z'_\infty\}$. Large deviation theory indicates that the event $\dist''_{kn'_0}=\bar \dist'^\star$ must occur with a nonzero probability.
  In particular, following Theorem 11.1.4 in \citet{cover1991information}, we have the lower bound
  \[
    \Prob(\dist''_{kn'_0}=\bar \dist'^\star)\geq \frac{\exp(-n'_0 k \KL(\bar \dist'^\star, \bar\dist''^\star))}{(k n'_0+1)^{n_0+1}}  \geq  \frac{\exp(-n'_0kr')}{(k n'_0+1)^{n_0+1}}  \quad \forall k\geq 1.
  \]
  However, if the event $\dist''_{kn'_0}=\bar \dist'^\star$ occurs then from Equation \eqref{eq:opt-form} it is clear that a poisoning adversary can perturb precisely $\alpha' k n'_0$ training data points to transform the training distribution $\dist''_{kn'_0}$ into a corrupted distribution, denoted here with $\dist_{kn_0'}$ that satisfies $\dist_{kn'_0}=\dist'_0$. Hence, it follows that with such a poisoning adversary which corrupts fewer than a fraction $\alpha$ of all data points we have $\Prob(\dist''_{kn'_0}=\bar \dist'^\star)\leq \Prob(\dist_{kn'_0}=\dist'_0)$. If however the event $\dist_{kn'_0}=\dist'_0$ occurs, we have by chaining the inequalities derived in the previous paragraph that $$\cost_\test(\theta_0)\geq \E_{\dist_\test''^\star}[\loss(\theta_0,Z)]\geq \cost_\HR ^{\cN,\alpha',r'}(\theta_0; \dist'_0)-\epsilon/4\geq \cost_\HR ^{\cN,\alpha,r}(\theta_0; \dist'_0)-\epsilon/2=\cost_A ^{\cN,\alpha,r}(\theta_0; \dist'_0)+\epsilon/2>\cost_A ^{\cN,\alpha,r}(\theta_0; \dist'_0).$$
  Consequently,
  \[
    \Prob(\cost_\test(\theta_0)>\cost_A ^{\cN,\alpha,r}(\theta_0; \dist'_0))\geq \Prob(\dist''_{kn'_0}=\bar \dist'^\star) \geq  \frac{\exp(-n'_0kr')}{(k n'_0+1)^{n_0+1}}  \quad \forall k\geq 1
  \]
  from which the result follows by passing to the limit $k\to\infty$ and recalling that $r'<r$.
\end{proof}

\section{Further details on experiments}
\label{app:experiments}

\subsection{Experimental Setup}\label{App: experimental_setup}

\textbf{MNIST} %
As stated in the main text, we use a CNN architecture specifically designed for MNIST and provided in the main Pytorch library 
\href{https://github.com/pytorch/examples/blob/main/mnist/main.py}{\faGithubSquare}.
This network is similar to the tensorflow variant originally used by \cite{madry2018towards}, consisting of two convolutional layers with 32 and 64 filters
respectively, followed by 2 × 2 max-pooling, and a fully connected layer of size 9216.
We use the ADAM optimizer with a learning rate of $1 \times 10^{-3}$ and without weight decay or dropout throughout the experiments. 

\textbf{CIFAR-10} %
For the majority of experiments in Section \ref{sec: experiments}, we use the ResNet18 architecture trained with the help of the ADAM optimizer with a starting learning rate of $1 \times 10^{-2}$ and without weight decay.
We use ADAM, rather than \ac{SGD}, to avoid need for manual tuning of the learning rate in our varied set of models and experiments.
In Section \ref{rob_of}, exceptionally, we use the same configurations as \cite{rice2020overfitting} for training.
That is, we use here \ac{SGD} with a starting learning rate of $0.1$ and a fixed learning rate decay schedule, as well as implementing their chosen weight decay of $5 \times 10^{-4}$.
We use the same learning rate decay schedule provided by \cite{rice2020overfitting}.
The attacks are carried out by a 10-step \ac{PGD} adversary.
For the robust overfitting experiments in Section \ref{rob_of}, we use an $\ell_\infty$ adversary of size $\epsilon = 8/255$ with a step size of $2/255$ mimicking the radius and step size considered by \cite{rice2020overfitting}.
For the holistic robustness experiments in Section \ref{HR}, we use an $\ell_2$ adversary of size $\epsilon = 25.5/255 = 0.1$.
This is smaller than, for instance, the attack size used by \cite{rice2020overfitting}, but we remark that such an adversary causes catastrophic overfitting in our setting, since we additionally face data poisoning and statistical error.
We use an adversarial step size of $0.2$, which is the default for the $\ell_2$ attack implemented in the \hyperlink{https://github.com/Harry24k/adversarial-attacks-pytorch}{torchattacks}
\href{https://github.com/Harry24k/adversarial-attacks-pytorch}{\faGithubSquare} library.

Except where stated otherwise, we train throughout all experiments for 300 epochs, which is longer than \cite{madry2018towards} and \cite{rice2020overfitting}.
We do so because in many cases we are working with a reduced number of batches (due to subsampling only a fraction of the available data), and hence the number of training steps scales down proportionally with the size of the subsampled dataset. We train for more epochs to ensure in particular convergence given fewer training batches per epoch.

\subsection{Benchmarking}
\label{app:benchmarking}

We benchmark our proposed \ac{HR} training method against other procedures. We describe here briefly how each of these procedures are ran and in particular how each of their hyperparameters are tuned.

\begin{itemize}
    \item \textbf{ERM}: there are no hyperparameters to tune.

    \item \textbf{PGD}: we validate over the parameter range $\epsilon \in \{0, 0.05, 0.1, 0.2\}$ associated with protection against $\ell_2$ evasion attacks. As stated before, we use an $\ell_2$ adversary with $\epsilon = 0.1$ to perturb the test data and hence the considered range contains the actual attack radius. When selecting the appropriate value for $\epsilon$ for HR training we consider the same range.


    \item \textbf{\ac{KL} DRO}: KL DRO counts just one hyperparameter $r$ which controls its robustness to statistical error. We consider here $r \in \{0, 0.05, 0.1, 0.2\}$. When selecting the appropriate value for $r$ for HR training we consider again the same range.

    \item \textbf{\ac{TV} DRO}: For TV DRO, there is also just one hyperparameter $\alpha$ controlling the robustness to data poisoning. We consider $\alpha \in \{0, 0.05, 0.1, 0.2\}$. When selecting the appropriate value for $\alpha$ for HR training we consider again the same range.

    \item \textbf{TRADES}: We use the $\ell_2$-adversary version of TRADES. We grid search over the two TRADES hyperparameters $\beta$ and $\epsilon$. We use $\beta \in \{1, 6\}$, the two values used in their comparison of defense methods, and $\epsilon \in \{0, 0.05, 0.1, 0.2\}$, the same range implemented for other methods.

    \item \textbf{DPA}: A single parameter $k$ controls the number of partitions over which each of the base classifiers are fit. We try $k \in \{5, 10, 50\}$. Since we are testing DPA on sub-sampled data, we avoid using larger $k$ as each partition is already comparatively smaller than when using the full dataset. Computing the training and testing loss for DPA (e.g. \autoref{fig: validation}) differs slightly compared to other methods, due to the nature of the DPA partitioning procedure. Training loss is taken as the average cross-entropy for each of the base classifiers. Testing loss is computed as follows: we first aggregate to find the average value of the output layer across the base classifiers. We then compute the cross-entropy loss of these outputs on the test set.

     \item \textbf{DPA + PGD}: 
     DPA splits the data into partitions, trains a separate model on each partition, and then ensembles the models to give the final prediction. 
     We combine DPA and PGD naturally by training the models on each partition with PGD. Hence, the overall training of DPA+PGD splits the data into partitions, train a robust model with PGD (as opposed to natural training) on each partition, and then ensemble the models.
     
     We cross-validate over the individual parameters used for DPA and PGD individually. That is $k \in \{5, 10, 50\}$ and $\epsilon \in \{0, 0.05, 0.1, 0.2\}$, giving 12 models in total.

    \item \textbf{HR}: We search over the grid $(\alpha, r, \epsilon) \in \{0, 0.05, 0.1, 0.2\}^3$, giving 4 options per robustness parameter and 64 models in total.
\end{itemize}

\subsection{Experimental Results}\label{App: experiments}

\subsubsection{Single parameter experiments}\label{App: single_param}

\begin{figure}[ht!]\label{fig: MNIST curves}
\centering
\begin{subfigure}
  \centering
  \includegraphics[width=0.4\columnwidth]{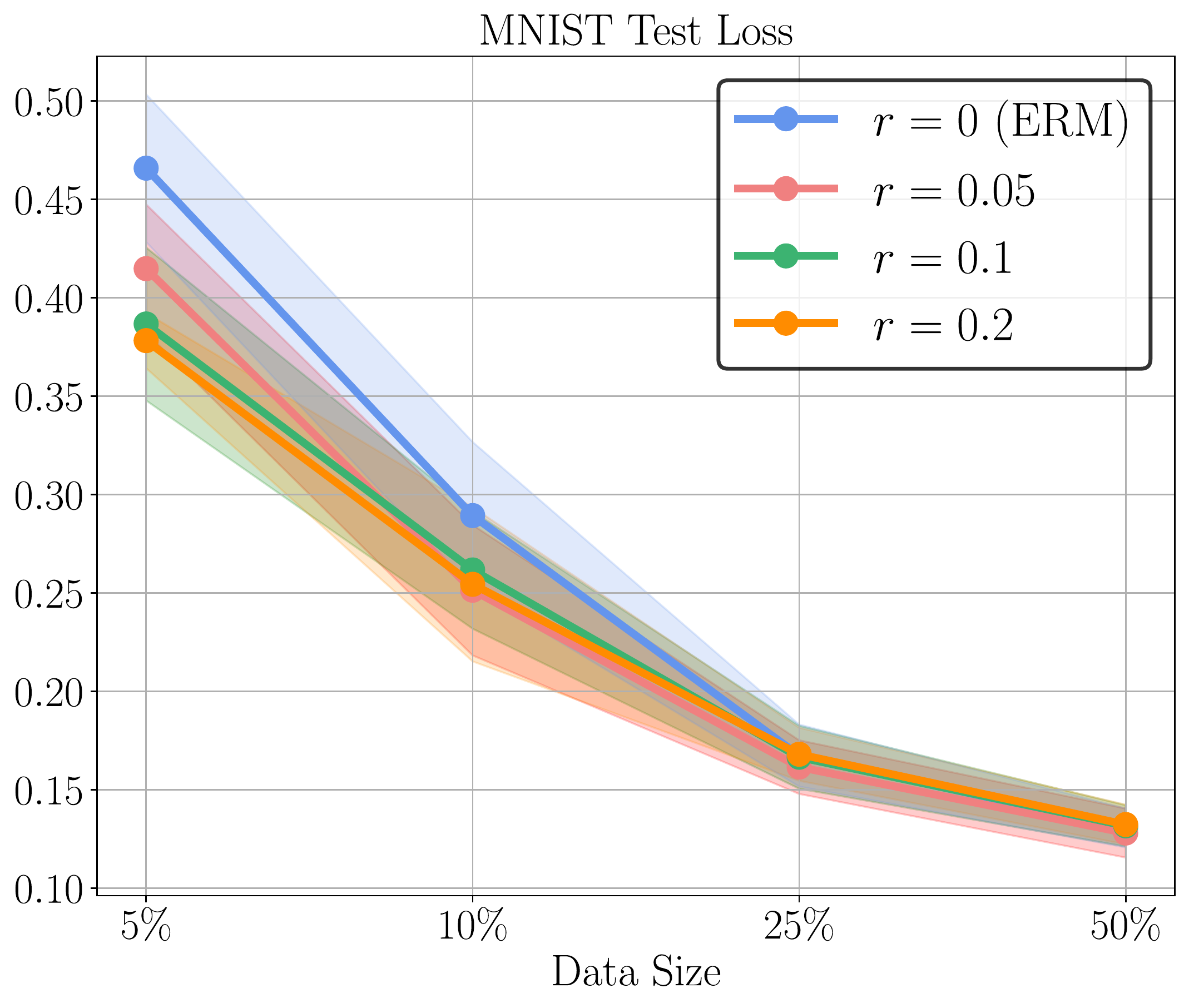}
\end{subfigure}%
\begin{subfigure}
  \centering
  \includegraphics[width=0.4\columnwidth]{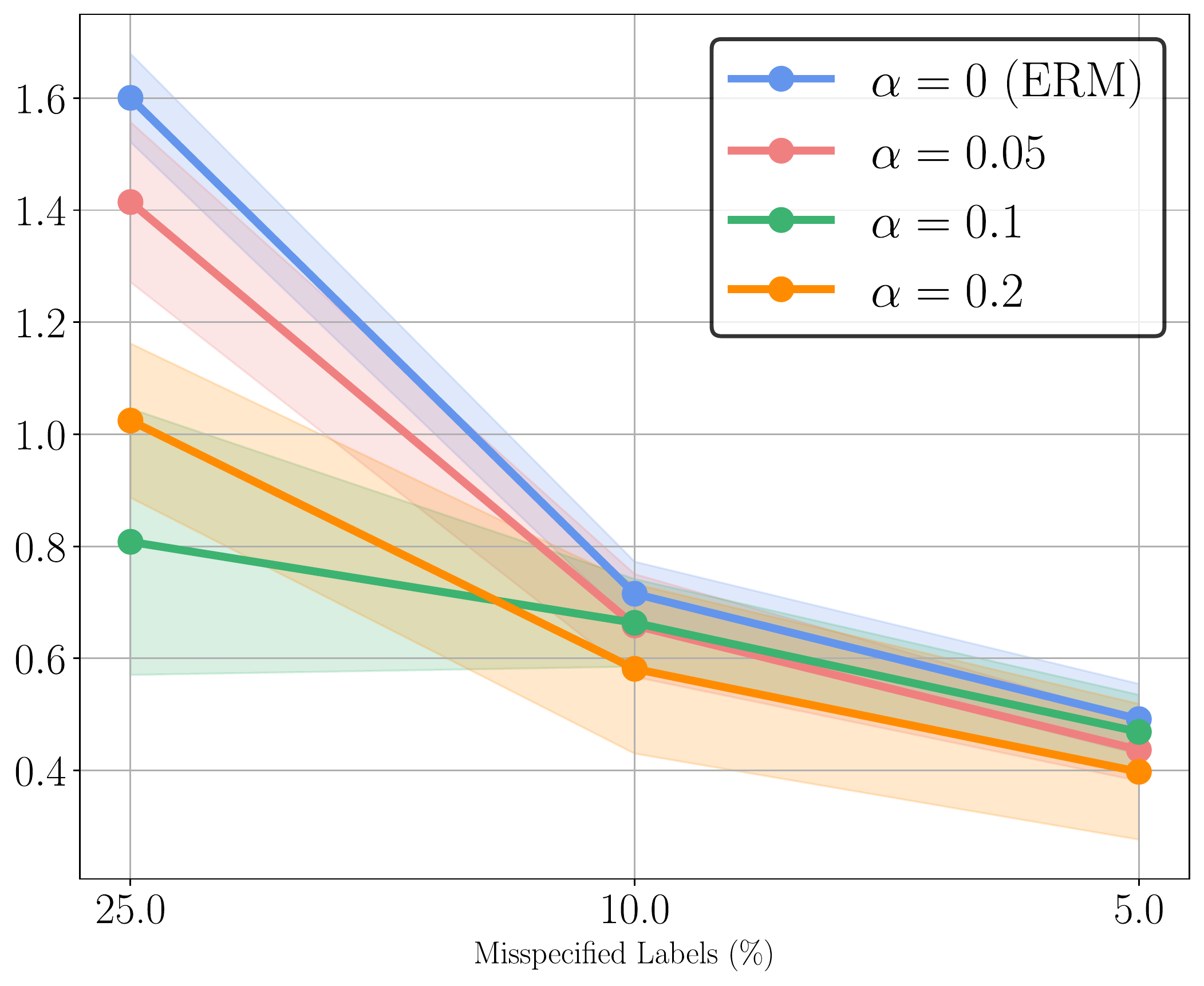}
\end{subfigure}%
\caption{Robustness to statistical error (left) and data poisoning (right) for MNIST. The same results are shown for CIFAR-10 in \autoref{fig:single_param} of the main text.}
\label{fig:single_param_mnist}
\end{figure}

\begin{figure}[ht!]\label{fig: r_u}
  \centering
  \includegraphics[width=1\textwidth]{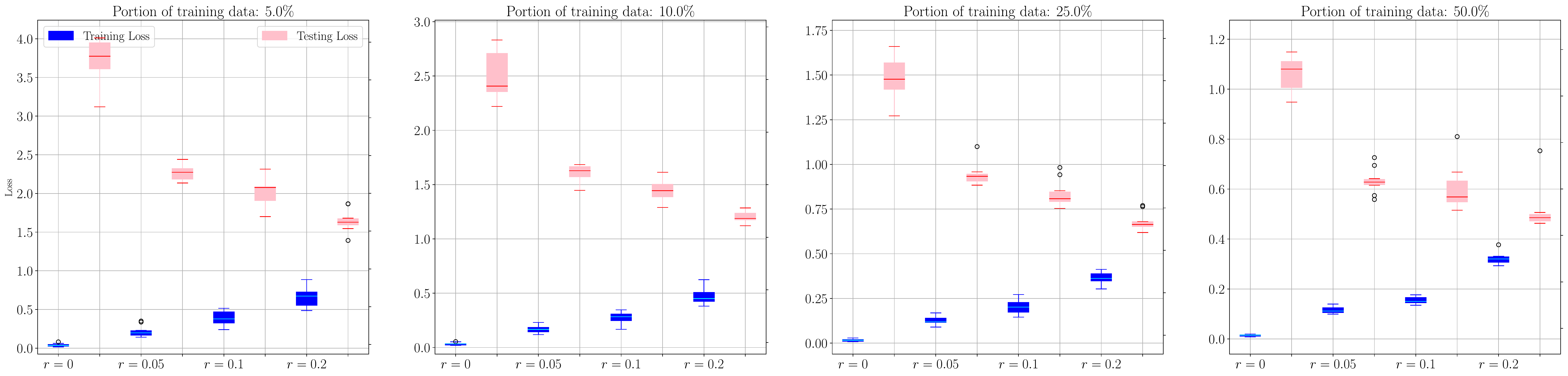}
  \caption{Experimenting HR with various values of $r$ when learning under statistical error on CIFAR-10. Here, $\alpha=0$ and $\cN = \{0\}$. See section \ref{sec:Evaluating Robustness} for setup.}
\end{figure}

\begin{figure}[ht!]\label{fig: alpha_u}
  \centering
  \includegraphics[width=1\textwidth]{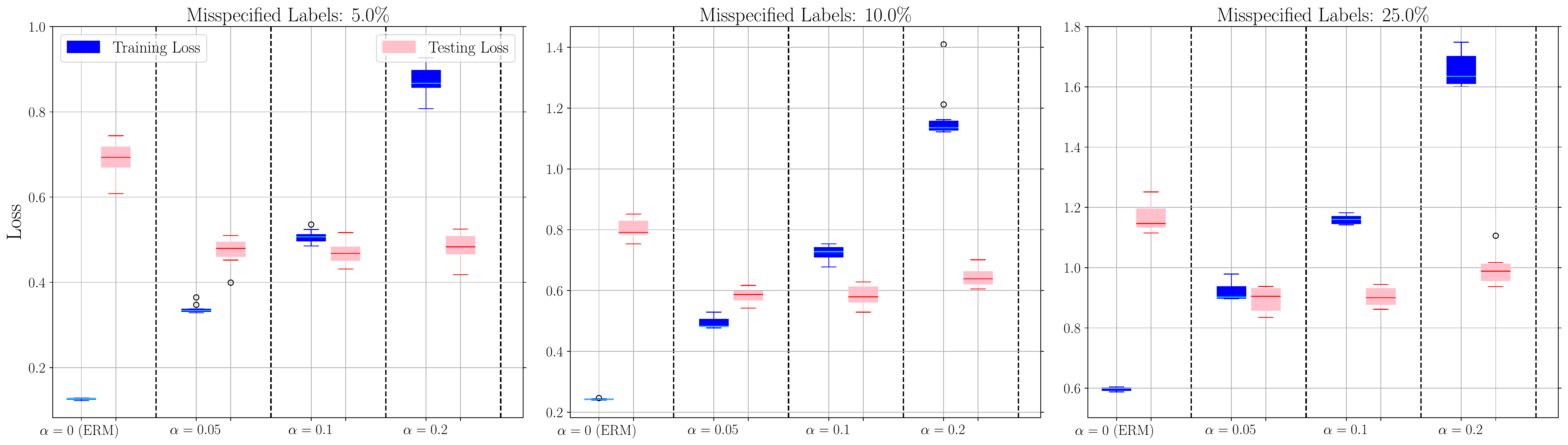}
  \caption{Experimenting HR with various values of $\alpha$ when learning under poisoning on CIFAR-10. Here, $r=0$ and $\cN = \{0\}$. See section \ref{sec:Evaluating Robustness} for setup.
  }
\end{figure}

\end{document}